\documentclass[11pt,reqno]{amsart}
\usepackage[a4paper, margin=1in]{geometry} 
\usepackage{amsmath,amsfonts,amssymb,mathtools,tensor}
\usepackage[round,authoryear]{natbib}
\usepackage{graphicx}
\usepackage{algorithm}
\usepackage{algpseudocode}
\usepackage{xcolor}
\usepackage{url}
\usepackage[pdfusetitle,hidelinks]{hyperref}
\mathtoolsset{showonlyrefs}

\newcommand{\epscritic}{\varepsilon_{\mathrm{crit}}}
\newcommand{\target}{\mathrm{target}}
\DeclareMathOperator*{\argmin}{arg\,min}
\DeclareMathOperator{\KL}{KL}
\DeclareMathOperator{\op}{op}

\theoremstyle{plain}
\newtheorem{lemma}{Lemma}
\newtheorem{theorem}{Theorem}

\newtheorem{corollary}{Corollary}
\theoremstyle{remark}
\newtheorem{remark}{Remark}
\theoremstyle{definition}

\newtheorem{assumption}{Assumption}
\newtheorem{definition}{Definition}

\usepackage{afterpage}
\makeatletter
\def\@setaffilfootnotes{\begingroup
  \def\author##1{}%
  \def\\{\unskip, \ignorespaces}%
  \def\curraddr##1##2{}\def\urladdr##1##2{}\def\email##1##2{}%
  \def\address##1##2{\@footnotetext{%
    \@ifnotempty{##1}{\textsuperscript{##1}\space}\textsc{##2}}}%
  \addresses
  \endgroup}
\def\emailsname{{\itshape E-mail addresses}}
\newif\if@firstemail
\def\@setemailfootnote{\begingroup
  \@firstemailtrue
  \def\author##1{}%
  \def\\{\unskip, \ignorespaces}%
  \def\address##1##2{}\def\curraddr##1##2{}\def\urladdr##1##2{}%
  \def\email##1##2{\@ifnotempty{##2}{%
    \if@firstemail \@firstemailfalse \emailsname\/:\space \else, \fi
    \texttt{##2}}}%
  \addresses
  \if@firstemail\else\@addpunct.\fi
  \endgroup}
\def\@adminfootnotes{%
  \let\@makefnmark\relax  \let\@thefnmark\relax
  \ifx\@empty\addresses\else
    \@setaffilfootnotes
    \@footnotetext{\@setemailfootnote}%
  \fi}
\AtBeginDocument{%
  \expandafter\gdef\expandafter\@savedkeywords\expandafter{\@keywords}}
\def\@setsavedkeywords{{\itshape \keywordsname.}\enspace
  \@savedkeywords\@addpunct.}
\def\@keywordsfootnote{\afterpage{\begingroup
  \let\@makefnmark\relax \let\@thefnmark\relax
  \@footnotetext{\@setsavedkeywords}\endgroup}}
\AtBeginDocument{}
\def\enddoc@text{\ifx\@empty\@translators \else\@settranslators \fi}
\makeatother

\title{An analysis of Mirror-Descent Soft Actor-Critic}

\author{Denis Zorba\textsuperscript{1}}
\address[1]{School of Mathematics, University of Edinburgh, United Kingdom}
\email{e.zorba@sms.ed.ac.uk}

\author{Michal Valko\textsuperscript{2}}
\address[2]{INRIA, Paris, France}
\email{michal.valko@inria.fr}

\keywords{Soft actor-critic, policy mirror descent, entropy regularisation,
  actor-critic, reinforcement learning}

\date{}

\begin{document}

\begin{abstract}
Soft Actor-Critic (SAC) is widely used for entropy-regularised reinforcement learning with continuous action spaces, and practical implementations perform only a few actor steps towards an evolving target. In this work, we prove convergence guarantees when the target policy arises from policy mirror descent and compare it with the classical Gibbs target. We derive sufficient conditions for the strong convexity and smoothness of the actor objective, characterised by the curvature of the $Q$-function estimate through the Legendre differential operator, and establish an $\mathcal{O}\!\left(N^{-\frac{1}{5}}\right)$ best-iterate finite-time convergence rate up to actor and critic approximation errors. Moreover, the mirror-descent step size $\lambda$ directly controls the target drift and hence actor tracking error, whereas the analogous Gibbs bound contains a non-vanishing tracking term.
\end{abstract}

\maketitle
\makeatletter\@keywordsfootnote\makeatother

\section{Introduction}
In reinforcement learning (RL), an agent seeks to learn an optimal policy that minimises its expected cumulative cost through interactions with an environment. This framework has led to remarkable successes across a range of challenging domains, including robotics, autonomous control and games \citep{levine2016end,lillicrap2015continuous,mnih2015human}. Broadly speaking, methods to tackle RL problems may be divided into two families: model-based approaches, which explicitly learn or exploit a model of the environment dynamics \citep{sutton_book,levine2013guided}, and model-free approaches, which seek to learn a policy or value function directly from observed transitions \citep{sutton_book,lillicrap2015continuous,pmlr-v80-haarnoja18b}.

In continuous action spaces, tackling model-free RL problems typically involves two learning problems. Firstly, equipped with function approximations, a critic (an estimate of the advantage or Q-function) is learned from observed transitions, commonly through some variant of temporal-difference learning \citep{sutton1988learning,watkins}, which is then in turn used to update the actor (the policy). Such actor updates may be performed directly through policy-gradient methods \citep{williams1992simple,sutton1999policy}, or through proximal methods which improve stability by controlling the change between successive policies \citep{schulman2015trust,schulman2017proximal}. A closely related idea arises from policy mirror descent \citep{lan2023policy,zhan2023policy,lan}, where a KL-divergence penalty is introduced directly into the policy-improvement problem, with the resulting update admitting a closed-form policy. However, in continuous action spaces, this target generally involves an intractable normalisation constant \citep{haarnoja2017reinforcement,pmlr-v80-haarnoja18b}.

Entropy regularisation leads to a number of theoretical and practical benefits~\citep{geist2019theory,cayci,david_leahy}. 
For example, entropy regularisation induces a unique statewise optimal policy under standard assumptions such as bounded costs \citep{geist2019theory,david_fisher} and can accelerate the convergence of policy gradient methods \citep{mei2020global,cen2022fast,lan2023policy}. It also encourages stochastic policies and persistent exploration \citep{haarnoja2017reinforcement,pmlr-v80-haarnoja18b}.

The Soft Actor-Critic (SAC) algorithm~\citep{haarnoja2018softapplications} was introduced as an efficient approach to tackling large-scale continuous model-free RL tasks with entropy regularisation, initially derived as a practical implementation of soft policy iteration \citep{pmlr-v80-haarnoja18b}. During SAC training, one performs gradient descent steps on a reverse Kullback–Leibler (KL) divergence with respect to some target policy, typically the exponential of the Q-function which in turn is motivated by the Bellman principle. Crucially, the normalisation constant is independent of the actor parameters and therefore disappears when taking gradients of the actor objective. This gives a practical policy-improvement procedure without requiring the intractable normalisation constant.

Despite the practical success of SAC, its theoretical analysis remains sparse and challenging. In practice, the actor is commonly parametrised by a squashed Gaussian distribution with a parametrised mean, that is a Gaussian followed by a $\tanh$ transformation to enforce the bounds of the action space (after scaling). Moreover, the reverse-KL projection onto the target policy is not solved exactly, but rather a single or a few steps of gradient descent are performed \citep{pmlr-v80-haarnoja18b}. Such an objective is in general non-convex, and the $\tanh$ transformation produces policy log-densities which become unbounded near the boundary of the action space. Therefore regularity hypotheses used in policy-gradient analyses can fail \citep{bedi2022hidden}, while policy-space convergence results \citep{lan2023policy,david_fisher} do not directly cover such a framework.

In this work, we provide the first theoretical guarantees for a population-level formalisation of the SAC framework, with a fixed-covariance squashed Gaussian actor, linear mean and a critic oracle, where one treats a policy arising from mirror descent as the target, as well as the classical Gibbs target. We consider the setting where only a single projected gradient step is performed towards an evolving target policy. Moreover we identify conditions on the curvature of the critic under which the reverse-KL actor objective becomes strongly convex and smooth, allowing us to control the error incurred by the actor as it tracks the changing target. 
\section{Related works}
There is a tremendous amount of research literature on RL theory, which underscores its importance. Our work relates to policy optimisation in entropy-regularised MDPs and to the analysis of actor-critic algorithms, which we summarise here.

\textbf{Entropy regularised RL.} 
Encouraging explorative policies during the training of RL policies have been demonstrated to be an important foundation of practical deep RL methods \citep{pmlr-v80-haarnoja18b, lan2023policy} which in turn establishes the importance of understanding entropy regularisation in MDPs. A particular line of work where the presence of entropy regularisation has been shown to have a meaningful contribution are policy mirror descent methods \citep{lan2023policy, lan, david_fisher}, where it was proven that the entropy regularisation can induce an exponential convergence to the optimal policy when the critic is known. More recently, it has been shown that one can maintain this exponential convergence when the critic is solved to sufficient accuracy in both the tabular case \citep{labbi2026refined} and the general case \citep{zorba2026mirrordescentactorcriticmethods}.

Having said that, such theoretical analyses typically rely on the updates existing in the policy space and using measure theoretic techniques. Moving from policy-space updates to a practical parametrisation, a line of work \citep{sherman, alfano2023novel, yuan2023linear, JMLR:v22:19-736} approximates mirror-descent updates by solving the minimisation sub-optimally under various assumptions.

\textbf{Actor-critic analysis.} Early convergence analyses use multi-timescale stochastic approximation and ODE based approaches \citep{borkar1997actor,konda}. Since then, the study of different regimes of actor critic algorithms can fall into a few categories. Firstly, a line of work studies the sample complexity of certain actor critic regimes in tabular and function-approximation settings under various assumptions in the unregularised case \citep{olshevsky2023small, kumar2024convergence, guar} and regularised case \citep{cayci_neural, labbi2026refined}. Another line of work studies the underlying convergence and stability behaviour of actor critic methods in general state and action spaces by focusing on a population level analysis \citep{zorba2026mirrordescentactorcriticmethods, stochastic_approx_application_AC, david_fisher, david_leahy}.

\section{Contributions}
\begin{itemize}

    \item We formalise the SAC policy class as the push-forward under $\tanh$ of a Gaussian distribution with parametrised mean, as standard in modern implementations of Soft Actor Critic \citep{pmlr-v80-haarnoja18b,haarnoja2018softapplications}. We study the reverse-KL objective for the actor parameters with target policies including the classical Gibbs measure \citep{pmlr-v80-haarnoja18b} and a policy arising from performing a step of policy mirror descent with step size/ penalty $\lambda$ \citep{tomar2021mirrordescentpolicyoptimization, lan2023policy}. Note that in general, such a loss function is non-convex. In this work, we identify sufficient conditions under which the strong convexity of such an objective is guaranteed on any Euclidean ball of radius $\mathrm{R} > 0$. We demonstrate that the curvature of the critic, quantified through the Legendre differential operator, plays a particular role alongside the entropy regularisation parameter $\tau > 0$.

    \item We prove global finite-time best-iterate convergence guarantees (up to approximation errors) under a $\varepsilon_{\mathrm{crit}}$-optimal critic oracle assumption. We derive bounds on the actor tracking error for both the classical Gibbs target and the mirror-descent target. We show that the mirror descent step size $\lambda $ can directly control the tracking error and provides an explicit trade off between balancing policy improvement against the actor's ability to track the target. In contrast, the corresponding bound for the classical Gibbs target contains a non-vanishing contribution determined by the regularisation parameter $\tau$. 

\end{itemize}

\section{Entropy Regularised MDPs}

Consider an infinite horizon Markov Decision Process $(S,A,P,c,\gamma)$, where $S$ is the state space (which may be an arbitrary Polish space) and $A$ is the action space, which we set as $A := (-1,1)^M$ for a fixed action dimension $M\geq1$. Moreover, $P\in \mathcal{P}(S | S \times A)$ is the state transition probability kernel, $c$ is a bounded cost function and $\gamma \in (0,1)$ is the discount factor. Let $\Lambda_{\mathbb R}$ denote Lebesgue measure on $\mathbb R$, let $\Lambda_{\mathbb R^M}:=\Lambda_{\mathbb R}^{\otimes M}$, and let $\Lambda$ denote normalised Lebesgue measure on $A$, that is $\Lambda(B):=2^{-M}\Lambda_{\mathbb R^M}(B)$ for $B\in\mathcal B(A)$. Finally, let $\tau>0$ denote a regularisation parameter. For each stochastic policy $\pi \in \mathcal{P}(A|S)$ and $s \in S$, we define the regularised value function by
\begin{equation}
\begin{aligned}\label{eq:value_function}
V^{\pi}_{\tau}(s)
&= \mathbb{E}_{s}^{\pi}\left(\sum_{n=0}^\infty\gamma^n
\Big(c(s_n,a_n) - \tau \mathrm{H}(\pi(\cdot|s_n))\Big)\right).
\end{aligned} 
\end{equation}
where $\mathbb{E}_{s}^{\pi}$ denotes the expectation over the state-action trajectory $(s_0,a_0,s_1,a_1,\ldots)$ generated by policy $\pi$ and kernel $P$ such that $s_0 := s,\ a_n \sim \pi(\cdot \mid s_n), s_{n+1} \sim P(\cdot|s_n,a_n)$ for all $n \geq 0$ and for any $\nu \in \mathcal{P}(A)$, the entropy is defined as 
\begin{equation}\label{eq:entropy_definition} 
\mathrm{H}(\nu)
:=
-\int_A\log\frac{d\nu}{d\Lambda}(a)\nu(da).
\end{equation}if $\nu$ is absolutely continuous with respect to the Lebesgue measure and $-\infty$ otherwise. As a result, for any $s \in S$ we have that $V^{\pi}_{\tau}(s) \in \mathbb{R}\cup \{+\infty\}$. Moreover for any state distribution $\rho \in \mathcal{P}(S)$, we denote $V^{\pi}_{\tau}(\rho) := \int_{S} V^{\pi}_{\tau}(s)\rho(ds)$. The optimal value function is defined as
\begin{equation}
\label{eq:optim}
\begin{aligned}
V^{*}_{\tau}(\rho) &= \inf_{\pi \in \mathcal{P}(A|S)}V^{\pi}_{\tau}(\rho).
\end{aligned}
\end{equation}
We refer to $\pi^*$ as the optimal policy if $V^*_{\tau}(\rho) = V^{\pi^*}_{\tau}(\rho)$. For each $\pi$, the state-action value function $Q^{\pi}_{\tau}$ is defined as
\begin{equation}
\label{eq:Q_func}
Q^{\pi}_{\tau}(s,a)=c(s,a)+\gamma\int_S V_{\tau}^{\pi}(s')P(ds'|s,a).
\end{equation}
Then for every $\pi\in\mathcal P(A|S)$ the advantage function is defined as $A^{\pi}_{\tau} = Q^{\pi}_{\tau}  + \tau \log \frac{d\pi}{d\Lambda} - V^{\pi}_{\tau}$. The occupation measure for policy $\pi$ with initial distribution $\rho$ is denoted by $d_{\rho}^{\pi}$ and is defined in Appendix \ref{sec:background}, alongside the Bellman principle for entropy-regularised MDPs.

\section{Soft Actor Critic and Mirror Descent}
In the seminal work of \citet{pmlr-v80-haarnoja18b}, a framework for tackling off-policy entropy regularised MDPs was introduced, which has since been regarded as one of the state of the art approaches for entropy regularised MDPs with continuous action spaces. The idea is relatively simple and consists of learning a critic through some variant of temporal difference/ Q-learning and then learning the new actor by projecting the exponential of the critic onto your policy class using the Kullback-Leibler (KL) divergence. That is, given some policy class $\Pi\subset \mathcal{P}(A|S)$, for each $s \in S$ compute
\begin{equation}\label{eq:KL_optimisation}
    \pi_{\text{new}} = \argmin_{\pi' \in \Pi} \KL\left(\pi'(\cdot|s) \bigg|\frac{1}{Z^{\pi_{\text{old}}}(s)}\exp\left(-\frac{1}{\tau}Q^{\pi_{\text{old}}}(s,\cdot)\right)\right).
\end{equation} One can then establish that \eqref{eq:KL_optimisation} guarantees policy improvement when the minimisation is solved fully, at least in the tabular setting \citep{pmlr-v80-haarnoja18b} and thus with exact critics, iterations of \eqref{eq:KL_optimisation} results in a sequence of policies which converges to an optimal policy inside $\Pi$, although it is worth mentioning that in most cases the true optimal policy $\pi^*$ need not belong in $\Pi$. Moreover, such guarantees only hold in the case where the optimisation in \eqref{eq:KL_optimisation} is solved to full accuracy and when the action space is finite.  

One issue that can arise from such a formulation is that the new policy $\pi_{\text{new}}$ may be drastically different from the old policy $\pi_{\text{old}}$, resulting in a large relative entropy $\KL(\pi_{\text{new}}(\cdot|s) | \pi_{\text{old}}(\cdot|s))$. Policy mirror descent provides a solution to this by penalising such policies. Formally, for some $\lambda > 0$, the policies are updated by
  \begin{equation}
  \begin{aligned}
  \pi_{\mathrm{MD}}(\cdot|s)
  &=
  \argmin_{m\in\mathcal P(A)}
  \Bigg\{
\int_A
  \left(
  Q_\tau^{\pi_{\mathrm{old}}}(s,a)
  +
  \tau\log
  \frac{d\pi_{\mathrm{old}}}{d\Lambda}(s,a)
  \right)
  m(da)+\frac{1}{\lambda}
  \KL\left(m\mid\pi_{\mathrm{old}}(\cdot|s)\right)
  \Bigg\}.
  \end{aligned}
  \label{eq:policy_mirror_descent}
  \end{equation}
  Suppose that the initial policy is uniform on $A$. By \citet{dupuis1997weak}, we have the following closed form solution to the optimisation problem
  \begin{equation}\label{eq:mirror_descent_old_new_closed_form}
  \frac{d\pi_{\mathrm{MD}}}{d\Lambda}(s,a)
  =
  \frac{1}{Z_{\mathrm{MD}}(s)}
  \left(
  \frac{d\pi_{\mathrm{old}}}{d\Lambda}(s,a)
  \right)^{1-\tau\lambda}
  \exp\left(
  -\lambda Q_\tau^{\pi_{\mathrm{old}}}(s,a)
  \right).
  \end{equation}
  From here, it is clear to see that the policy mirror descent policy behaves as a geometric interpolation between the old policy and the full soft policy-improvement update, controlled by the proximal penalty/ step size $\lambda >0$ and the entropy regularisation $\tau > 0$. In particular, they coincide when precisely we have $\lambda = \frac{1}{\tau}$, giving the soft policy-improvement update of \citep{pmlr-v80-haarnoja18b}.

    Given that computing the normalisation constant in \eqref{eq:mirror_descent_old_new_closed_form} is just as intractable as \eqref{eq:KL_optimisation}, one natural extension is to replace the target in \eqref{eq:KL_optimisation} with \eqref{eq:mirror_descent_old_new_closed_form}, since the normalisation constant does not affect the gradient steps. This approach has demonstrated promising empirical results \citep{tomar2021mirrordescentpolicyoptimization}.


\section{Policy Class}

Let us now introduce the policy class before presenting the algorithm. Recall that $A = (-1,1)^M$ and $\mathfrak{B}(\mathrm{R}) = \left\{w \in \mathbb{R}^d : |w|_2 \leq \mathrm{R} \right\}$ is the Euclidean ball of radius $\mathrm R$ in $\mathbb{R}^d$. Let $x:S\to\mathbb R^{d\times M}$ be a bounded feature map satisfying $|x(s)|_{\op}\leq1$ for all $s\in S$, where $|\cdot|_{\op}$ is the Euclidean operator norm (see Appendix \ref{sec:notation} for an overview of notation). 
For any $\varsigma>0$, we denote the centred Gaussian with covariance $\varsigma^2I_M$ by $q^{\varsigma}:=\mathcal N(0,\varsigma^2I_M)$. Let $\Sigma:=\operatorname{diag}(\sigma_1^2,\ldots,\sigma_M^2)$ with $\sigma_i>0$ for all $i$. For any $w\in\mathfrak B(\mathrm R)\subset\mathbb R^d$, we parametrise the mean by $m_w(s)=x(s)^\top w\in\mathbb R^M$ for each $s\in S$ and denote $q_w(\cdot|s):=\mathcal N(m_w(s),\Sigma)$. We denote this class of Gaussian kernels by $\mathcal{Q}_\Sigma:=\{q_w:w\in\mathfrak B(\mathrm R)\}\subset\mathcal P(\mathbb R^M|S)$. Throughout, $\tanh$ and its inverse act component-wise. To formalise the policy class, we must first introduce the notion of a push forward measure.

\begin{definition}[Push-forward measure, {\citealp[\S5.2]{ambrosio2008gradient}}]
\label{def:pushforward-pullback}
Let \(T:\mathbb R^M\to A\) be measurable and let \(\nu\in\mathcal P(\mathbb R^M)\).
The push-forward of \(\nu\) by \(T\) is the probability measure
\(T_{\#}\nu\in\mathcal P(A)\) such that for any $B\in\mathcal B(A)$,
\begin{equation}
        T_{\#}\nu(B)
        :=
        \nu(T^{-1}(B)).
\end{equation}
Equivalently, for every bounded measurable function \(f:A\to\mathbb R\),
\begin{equation}
        \int_A f(a)\,T_{\#}\nu(da)
        =
        \int_{\mathbb R^M} f(T(u))\,\nu(du).
\end{equation}
\end{definition}
The policy class is then simply the push-forward under $\tanh$ of $\mathcal{Q}_\Sigma$.
  \begin{definition}[Squashed Gaussian Policies]
 The admissible policy class is given by
  \[
          \Pi_\Sigma
          :=
          \left\{
          \pi_{w}=(\tanh)_{\#}q_w
          :
          q_w\in\mathcal Q_{\Sigma}
          \right\}.
  \]
  \end{definition}

The following lemma of Dupuis-Ellis becomes very useful in the proofs of the main results.
  \begin{lemma}[{\citealp[Lemma E.2.1]{dupuis1997weak}}]
  \label{lem:dupuis-ellis}
  Let $T:\mathbb R^M\to A$ be a measurable bijection with measurable inverse. Then for any
  $\nu,\eta\in\mathcal P(\mathbb R^M)$,
  \[
          \operatorname{KL}\left(
          T_{\#}\nu |
          T_{\#}\eta
          \right)
          =
          \operatorname{KL}(\nu |\eta).
  \]
  \end{lemma}
Note that in general the log densities $\log \frac{d\pi_w}{d\Lambda}$ are unbounded near the boundary of the action space $A = (-1,1)^M$. Therefore, one cannot directly apply classical tools such as the performance difference lemma for entropy regularised MDPs, which require bounded log densities, without careful treatment. However one can show that although the log densities are unbounded, the entropy $\mathrm{H}(\pi_w(\cdot|s))$ remains bounded for all $s \in S$ when $w \in \mathfrak{B}(\mathrm{R})$, which can serve as a remedy (see Lemma \ref{lem:uniform_entropy_bound}). Furthermore, while we adopt the $\tanh$ transformation due to its tractability and widespread use in practical implementations of soft actor-critic, much of the analysis extends to other sufficiently regular bijections from $\mathbb{R}^M$ to $(-1,1)^M$.
\section{Algorithm}
For any iteration $n \in \mathbb{N}$, we denote the replay buffer state law by $\rho_n \in \mathcal{P}(S)$ with $\rho_0 = \rho$ the initial state distribution, which we model as an exponentially weighted mixture of all preceding state-occupancies up to the iteration $n > 0 $ with parameter $\chi > 0$ (see Definition \ref{def:ema_replay} in Appendix \ref{sec:background}). Then for any target policy $\pi_{\target} \in \mathcal{P}(A|S)$, $\pi_w \in \Pi_\Sigma$ and $n \in \mathbb{N}$ we define 
\begin{equation}
    J_n(w,\pi_{\target}) = \int_{S} \operatorname{KL}(\pi_w(\cdot|s) | \pi_{\target}(\cdot|s))\rho_n(ds),
\end{equation}Now suppose that one has access to an estimate of the Q-function for each encountered policy, and denote this sequence by $\{\widehat{Q}^n \}_{n\geq 0}$. Let $\pi^0(da|s):=\Lambda(da)$ for all $s \in S$. Firstly, we will denote the Gibbs target as
\begin{equation}\label{eq:classic_target}
    \frac{d\pi^{n+1}_\mathrm{G}}{d\Lambda}(s,a)  = \frac{1}{Z_n(s)} \exp \left(-\frac{1}{\tau}\widehat{Q}^{n}(s,a) \right).
\end{equation}
Furthermore, we denote the target arising from policy mirror descent as
\begin{equation}\label{eq:mirror_descent}
\begin{aligned}
\pi^{n+1}(\cdot|s)
&=\argmin_{m\in\mathcal P(A)}\Bigg\{\int_A\left(\widehat Q^n(s,a)+\tau\log\frac{d\pi^n}{d\Lambda}(s,a)\right)m(da)+\frac1\lambda\operatorname{KL}(m|\pi^n(\cdot|s))\Bigg\}.
\end{aligned}
\end{equation}
Throughout, $\{\pi^n\}_{n\geq0}$ denotes the evolving target sequence, updated recursively from the previous target rather than from the current actor, whereas $\pi_{w^n}$ denotes the actor tracking it. Moreover, throughout we take the initial actor satisfies $w^0\in\mathfrak B(\mathrm R)$, and the replay parameters satisfy $0<\chi\lambda\leq1$. We present the training procedure in Algorithm \ref{algo:SAC_fixed_replay}. To summarise, we study the case where one performs a single step of gradient descent of the objective with a mirror descent target for each critic estimate, then projects the new parameters onto $\mathfrak{B}(\mathrm{R})$.
  \begin{algorithm}[h]
      \caption{Projected actor update}
      \label{algo:SAC_fixed_replay}
      \label{alg:projected-sac-mstep}
      \begin{algorithmic}[1]
      \For{\(n=0,1,2,\ldots\)}
          \State Obtain critic estimate $\widehat Q^{n}$  
          \State Update replay law $\rho_{n+1}$ 
          \State \(\displaystyle
          w^{n+1}
          =
          \mathcal P_{\mathfrak B(\mathrm{R})}
          \left(
          w^n
          -
          h\nabla_w J_{n+1}(w^n,\pi^{n+1})
          \right)
          \)
      \EndFor
      \end{algorithmic}
  \end{algorithm}

We work under the following assumptions.

\begin{assumption}[$Q$-function oracle]\label{ass:Q_oracle}
There exists $\varepsilon_{\mathrm{crit}}\geq 0$ such that, for all $n\geq0$, the critic estimate $\widehat Q^n\in B_b(S\times A)$ satisfies $\widehat{Q}^n(s,\cdot) \in C^2(A)$ for each $s \in S$, $\left|\widehat{Q}^n\right|_{B_b(S\times A)} \leq \widehat{Q}_{\mathrm{max}}$ and 
\begin{equation}\label{eq:Q_oracle}
\left|\widehat Q^n-Q_\tau^{\pi_{w^n}}\right|_{B_b(S\times A)}
\leq\varepsilon_{\mathrm{crit}}.
\end{equation}
\end{assumption}

Assumption \ref{ass:Q_oracle} allows us to isolate dynamics of training the Gaussian policies $\pi_{w^n}$. While one could additionally introduce a parametrised critic and analyse its update jointly with the actor, we abstract away from the critic dynamics and capture the resulting approximation error through $\varepsilon_{\mathrm{crit}}$. This further allows our results to accommodate general twice-differentiable \(Q\)-function parametrisations, including standard smooth neural network architectures. Assumption \ref{ass:Q_oracle} holds in many practical implementations when one freezes the policy and performs sufficiently many steps of temporal difference learning and when the parametrisation of the critic is sufficiently rich.

\begin{assumption}[Optimal-occupancy concentrability]\label{as:optimal_occupancy}
There exists $A_2<\infty$ such that
\begin{equation}
    \left|\frac{d d_\rho^{\pi^*}}{d\rho}\right|_{L^2(\rho)}\leq A_2.
\end{equation}
\end{assumption}
Assumption \ref{as:optimal_occupancy} is designed to control the distributional shift between the optimal discounted occupancy measure and the initial state distribution. Such concentrability or distribution-mismatch
  conditions are standard in analyses of approximate dynamic programming,
  off-policy learning and policy-gradient methods
  \citep{munos2008finite,chen2019information,JMLR:v22:19-736}. Note that, we require that only the optimal discounted occupancy measure $d_\rho^{\pi^*}$ has a density with respect to the initial state distribution $\rho\in \mathcal{P}(S)$ and is square integrable with respect to $\rho$ which is strictly weaker than imposing a uniform bound in the supremum norm \citep{JMLR:v22:19-736} and imposing a similar assumption for every encountered policy during training \citep{pmlr-v119-uehara20a,munos2008finite}. One case where Assumption \ref{as:optimal_occupancy} holds is when the transition kernel is square integrable with respect to the initial state distribution (see Lemma \ref{lemma:occupancy_L2_bounds}). This includes, under appropriate covariance and mean conditions, dynamics with Gaussian transition noise when the initial state distribution is also Gaussian. 
\begin{assumption}[Eigenvalues for actor parametrisation]\label{as:policy_Evalues}
  Assume that
  \begin{equation}
      \lambda_P
      :=
      \lambda_{\min}\left(\int_S x(s)x(s)^\top\rho(ds)\right)
      >
      0.
  \end{equation}
\end{assumption}
Assumption \ref{as:policy_Evalues} is a standard non-degeneracy condition on the feature map and requires that the feature covariance matrix under the initial state distribution be positive definite, ensuring that every direction in the actor parameter space is sufficiently represented under only the initial state distribution $\rho$.

\section{Main results}
While we do not impose any assumptions regarding the realisability of the critic or policy in the main results, we inevitably must account for the bias that arises when the policy class of squashed Gaussians are not sufficiently rich enough to represent the target policies. 
\begin{definition}[Gaussian approximation error]\label{def:gaussian_approximation_error}
For each $\{\widehat{Q}^n\}_{n\geq0}$, the Gaussian approximation error to the mirror descent policies $\{\pi^n\}_{n\geq 0}$ is defined as
\begin{equation}
    \varepsilon_{\mathrm{act}}^n  =  \min_{w \in  \mathfrak{B}(\mathrm{R})}J_n(w,\pi^n).
\end{equation}
Moreover, for any $n\geq1$, we define $\varepsilon_{\mathrm{approx}}(n)=\epscritic+\frac{1}{n}\sum_{k=0}^{n-1}\left(\varepsilon_{\mathrm{act}}^k\right)^{\frac12}$
\end{definition}

In the following results, we denote $w^n_* \in \argmin_{w \in \mathfrak{B}(\mathrm{R})}J_n(w,\pi^{n})$, the optimal parameters $w \in \mathfrak{B}(\mathrm{R})$ for the target at iteration $n \in \mathbb{N}$. Theorem \ref{thm:main_final_bound} establishes the first bound on the value function optimality gap for Algorithm \ref{algo:SAC_fixed_replay}.

\begin{theorem}\label{thm:conv_1}\label{thm:main_final_bound}
Let Assumptions \ref{ass:Q_oracle} and \ref{as:optimal_occupancy} hold, let $0 < \tau\lambda < 1$. Then there exists $\mathrm C<\infty$ such that for all $n \geq 1$,
    \begin{equation}
    \begin{aligned}
        &\min_{0 \leq r \leq n-1} V^{\pi_{w^{r}}}_{\tau}(\rho) - V^{\pi^*}_{\tau}(\rho)\\
        &\quad\leq \mathrm C\left(\lambda\left(\frac{1}{\lambda^2 n}+1\right)
        +\frac1n\sum_{k=0}^{n-1}\left(J_k(w^k,\pi^k)-J_k(w_*^k,\pi^k)\right)^{\frac12}
        +\varepsilon_{\mathrm{approx}}(n)\right).
    \end{aligned}
    \end{equation}
\end{theorem}
See Appendix \ref{sec:proof_of_thm_1} for a proof.

Theorem \ref{thm:conv_1} shows that, up to approximation errors, the convergence of Algorithm \ref{algo:SAC_fixed_replay} relies on two components. The first is the actor tracking error, that is how well can performing just one step of gradient descent on the objective track the evolving target. Secondly, one can see that the first term yields a constant which can be controlled by $\lambda$. Thus if one can also control the actor tracking error by a function of $\lambda $, after a careful fine tuning a convergence rate may be possible. 
To that end, as standard in the analyses of coupled actor-critic and online gradient descent algorithms \citep{barakat22a-linear,olshevsky2023small, mokhtari2016online}, we must establish sufficient regularity of the objective in order to control the tracking error.

To that end, since in general the reverse KL objective is non-convex, we must derive some sufficient conditions under which convexity holds. The first step towards doing so is by defining the following operator.
\begin{definition}\label{def:differential_operator}
Let $D:A\to\mathbb S^M$ be $D(a):=\operatorname{diag}(1-a_1^2,\ldots,1-a_M^2)$ and for any $f\in C^2(A)$ define $\mathcal T:C^2(A)\to C(A;\mathbb S^M)$ by
\begin{equation}
(\mathcal T f)(a)
:=
D(a)^{\frac{1}{2}}\nabla_a^2 f(a)D(a)^{\frac{1}{2}}
-2\operatorname{diag}\bigl(a_1\partial_{a_1}f(a),\ldots,a_M\partial_{a_M}f(a)\bigr).
\end{equation}
\end{definition}
Firstly observe that since $f\in C^2(A)$, $\mathcal T f$ is continuous on $A$ and thus $\mathcal T$ is well defined. Moreover for any $a\in A$, we can express $\mathcal{T}$ component-wise as
\begin{equation}
(\mathcal T f)_{ij}(a)
=
\begin{cases}
(1-a_i^2)\partial_{a_i a_i}f(a)-2a_i\partial_{a_i}f(a),&i=j,\\
\sqrt{(1-a_i^2)(1-a_j^2)}\,\partial_{a_i a_j}f(a),&i\ne j.
\end{cases}
\end{equation}
For the special case of $M=1$, $\mathcal{T}$ reduces to the classical Legendre differential operator $\frac{d}{da}((1-a^2)f'(a))$ \citep[\S18.8]{dlmf}. In the present setting, $\mathcal T$ provides a convenient way of expressing the curvature of the actor objective and critic estimate and which in turn allows us to establish sufficient convexity conditions on the actor objective.
  Recall that, for each integer $k\geq1$, $\mathbb S^k$ denotes the space of real symmetric matrices in $\mathbb R^{k\times k}$.

  \begin{theorem}\label{thm:strong_convexity_L_smoothness_objective}Let Assumptions~\ref{ass:Q_oracle} and~\ref{as:policy_Evalues} hold and define
      \begin{equation}
          R_{\widehat Q}
          :=
          \sup_{n\in\mathbb N,\ s\in S,\ a\in A}\left|\mathcal T(\widehat Q^n(s,\cdot))(a)\right|_{\op},
      \end{equation}
      Moreover let $0 < \tau\lambda < 1$ and $c_{\Sigma,\mathrm R}:=\min_{1\leq i\leq M}\left(1-\tanh^2(\mathrm R+\sigma_i)\right)(2\Phi(1)-1)$
      with $\Phi : \mathbb{R} \to (0,1)$ the standard Gaussian cumulative distribution function. If it holds that
      \begin{equation}\label{eq:fixed_sigma_curvature_condition}
          \frac{R_{\widehat Q}}{2}<\tau,
      \end{equation}
      then the functions $w\mapsto J_n(w,\pi^n)$ and $w\mapsto J_n(w,\pi_{\mathrm G}^n)$ are $(1-\gamma)\left(2-\frac{R_{\widehat Q}}{\tau}\right)c_{\Sigma,\mathrm R}\lambda_P$-strongly convex and $2+\frac{R_{\widehat Q}}{\tau}$-smooth on $\mathfrak{B}(\mathrm R)$ for all $n \in \mathbb{N}$.
  \end{theorem}
See Appendix \ref{sec:proof_of_convexity_conditions} for a proof.

Theorem \ref{thm:strong_convexity_L_smoothness_objective} demonstrates sufficient conditions under which for any $n \in \mathbb{N}$, the actor objective with mirror descent target $w\mapsto J_{n}(w,\pi^{n})$ is strongly convex on the Euclidean ball of radius $\mathrm{R} > 0$. Fundamentally speaking, this condition ensures that the Hessian of the actor objective is uniformly positive definite. Intuitively speaking, the curvature induced by the differential entropy relative to Lebesgue measure dominates the worst-case adverse curvature introduced by the critic, thereby preventing the actor objective from developing locally non-convex directions. See Remark \ref{rem:linear_fourier_critic} in Appendix \ref{sec:proof_of_convexity_conditions} for an explicit example of $R_{\widehat{Q}}$ for Fourier critic parametrisations and see Figure \ref{fig:critic_curvature_performance} for a comparison of training dynamics for the MDP defined in Appendix \ref{sec:implementation_details} in the cases where the condition in Theorem \ref{thm:strong_convexity_L_smoothness_objective} is satisfied and violated.

With Theorem \ref{thm:strong_convexity_L_smoothness_objective} in hand, we now seek to control the actor tracking error.

\begin{lemma}[Actor tracking]\label{lemma:actor_tracking}
Let Assumptions \ref{ass:Q_oracle} and \ref{as:policy_Evalues} and the condition of Theorem \ref{thm:strong_convexity_L_smoothness_objective} hold, and suppose that $0<h<\frac{\tau}{2\tau+R_{\widehat Q}}$.
Then there exists a non-negative constant $C_1<\infty$ such that for every $n\geq1$, it holds that
\begin{align}
    \frac{1}{n}
    \sum_{k=0}^{n-1}
    \left(
    J_k(w^k,\pi^k)-J_k(w_*^k,\pi^k)
    \right)^{\frac{1}{2}}
    &\leq
    C_1
    \left(
    \frac{1}{h n}
    +
    \frac{\lambda }{h^2}
    \right)^{\frac{1}{4}}.
\end{align}
\end{lemma}
See Appendix \ref{sec:proof_of_actor_tracking} for a proof.

Lemma \ref{lemma:actor_tracking} separates the actor optimisation error from the drift of the target and shows that, when the target policy arises from the mirror-descent update \eqref{eq:mirror_descent}, the drift of the target can be directly controlled through the mirror-descent step size $\lambda>0$. In particular, small $\lambda$ makes consecutive target policies closer and therefore allows the actor, despite performing only a single gradient step per iteration, to track the evolving target more accurately.

However, of course choosing $\lambda$ arbitrarily small does not come for free. Indeed, Theorem \ref{thm:main_final_bound} contains the term $\frac{1}{\lambda n}$, reflecting the progress made by the underlying mirror-descent updates that we are approximating. This demonstrates an intrinsic trade off for Algorithm \ref{algo:SAC_fixed_replay} between smaller values of $\lambda$, which produce a slowly moving target that is easier for the actor to track, and larger values of $\lambda$ which allow for more aggressive policy improvement but increase the actor tracking error. Therefore balancing these effects is required to establish a result which demonstrates convergence up to approximation errors.

It is now worth addressing the behaviour of the actor tracking error when one has the classical Gibbs target \eqref{eq:classic_target} rather than the mirror descent target. For the following corollary and its proof, $\{w^n\}_{n\geq0}$ denotes the iterates of Algorithm~\ref{algo:SAC_fixed_replay} with $\pi^{n+1}$ replaced by $\pi_{\mathrm G}^{n+1}$ in the actor update in Algorithm \ref{algo:SAC_fixed_replay}. To that end, for $k\geq1$ let us denote $w_{\mathrm G,*}^k=\argmin_{w\in\mathfrak B(\mathrm R)}J_k(w,\pi_{\mathrm G}^k)$, which is unique under the conditions of Theorem \ref{thm:strong_convexity_L_smoothness_objective}.

\begin{corollary}
\label{cor:gibbs_actor_tracking}
Let Assumptions \ref{ass:Q_oracle} and \ref{as:policy_Evalues} and the condition of Theorem \ref{thm:strong_convexity_L_smoothness_objective} hold and let $0<h<\frac{\tau}{2\tau+R_{\widehat Q}}$. Then there exists $C_{\mathrm G}<\infty$ such that, for every $n\geq1$,
\begin{align}
&\frac1n\sum_{k=1}^{n-1}
\left(J_k(w^k,\pi_{\mathrm G}^{k})-J_k(w_{\mathrm G,*}^{k},\pi_{\mathrm G}^{k})\right)^{\frac{1}{2}}\leq C_{\mathrm G}
\left(\frac{1}{hn}+\frac{1}{h^2}\left(\frac{1}{\tau}+\chi\lambda\right)\right)^{\frac{1}{4}}.
\end{align}
\end{corollary}

The classical Gibbs target incurs the non-vanishing term $\frac{1}{\tau}$. Consequently, even when the replay distribution changes slowly, the bound in Corollary \ref{cor:gibbs_actor_tracking} does not guarantee that the actor can track the Gibbs target arbitrarily closely. In contrast, the mirror-descent formulation provides an explicit mechanism for slowing the evolution of the target to match the optimisation timescale of the actor. See Figure \ref{fig:actor_tracking_lambda}.

\begin{theorem}\label{thm:tuned_convergence_rate}
Let Assumptions \ref{ass:Q_oracle}, \ref{as:optimal_occupancy}, and \ref{as:policy_Evalues} hold and suppose that the conditions of Theorem~\ref{thm:strong_convexity_L_smoothness_objective} hold. For any integer $N\geq1$, let $\vartheta:=\min\left\{\frac{1}{\tau},\frac{1}{\chi}\right\}$ and let $\left\{\pi_{w^r}\right\}_{r=0}^{N-1}$ be the iterates of Algorithm \ref{algo:SAC_fixed_replay} with $\lambda = \frac{\vartheta}{2N^{\frac{4}{5}}}$ and $0 < h <\frac{\tau}{2\tau+R_{\widehat Q}}$. Then there exists $C<\infty$ such that
\begin{align}
    \min_{0\leq r\leq N-1}
    V_\tau^{\pi_{w^r}}(\rho)-V_\tau^{\pi^*}(\rho)
    \leq
    C\left(N^{-\frac15}+\varepsilon_{\mathrm{approx}}(N)\right).
\end{align}
\end{theorem}
Theorem \ref{thm:tuned_convergence_rate} demonstrates that for a carefully tuned proximal penalty/step-size $\lambda > 0$ and step size $h > 0$ and any projecting radius $\mathrm{R}> 0$ one can establish a finite time convergence rate up to approximation errors. One promising route to achieving a faster convergence rate, at the cost of a more intricate analysis, is by also incorporating the critic dynamics and controlling the accuracy to which the critic is solved to at each iteration. Another one is to restrict $\mathrm{R} > \delta$ for some $\delta > 0$ sufficiently large in order to ensure that the minimisers $w_*^n$ lie in the interior of $\mathfrak{B}(\mathrm{R})$. We do not pursue this here in order to retain generality for every $\mathrm{R}>0$.

\begin{figure}[!ht]
\centering
\includegraphics[width=\linewidth]{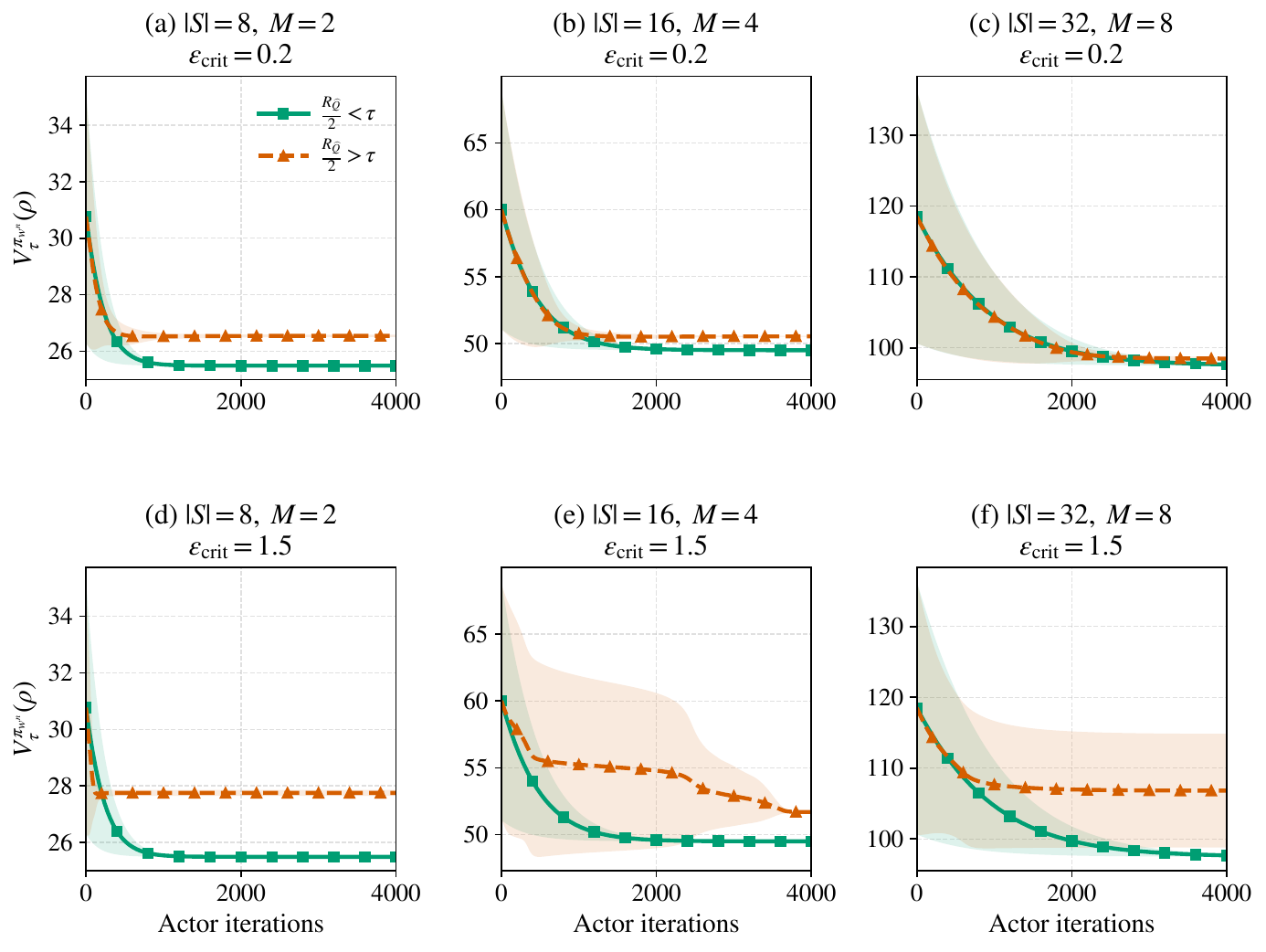}
\caption{Training curves of Algorithm~\ref{algo:SAC_fixed_replay} for the MDP in Appendix~\ref{sec:implementation_details}, with varying state-space size $|S|$ and action dimension $M$, $A=(-1,1)^M$, and temperature $\tau=1$. Within each panel, the two critics have equal error $\varepsilon_{\mathrm{crit}}$ but different $R_{\widehat Q}$. The curvature condition of Theorem~\ref{thm:strong_convexity_L_smoothness_objective} is satisfied in green and violated in orange. Curves show means and shaded regions show one standard deviation across nine paired initialisations.}
\label{fig:critic_curvature_performance}
\end{figure}

\begin{figure}[!ht]
\centering
\includegraphics[width=\linewidth]{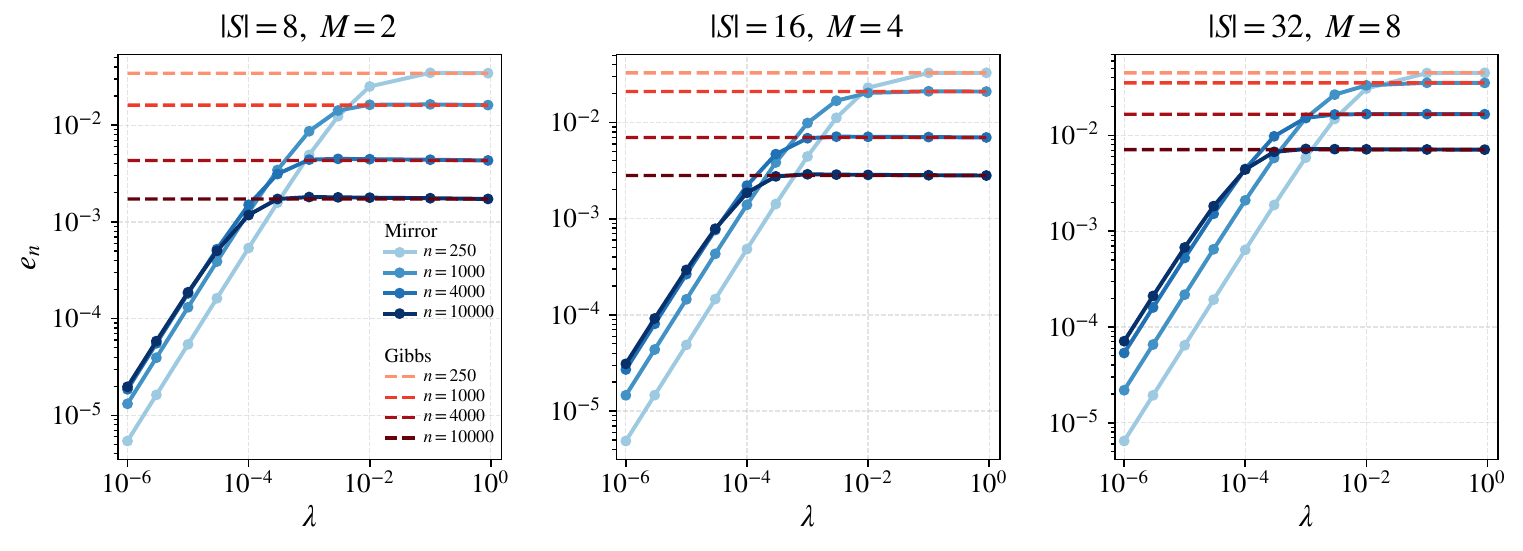}
\caption[Actor tracking error against the mirror-descent step size]{Averaged actor tracking error $ e_n:=\frac1n\sum_{k<n}\left(J_k(w^k,\pi^k)-J_k(w_*^k,\pi^k)\right)^{\frac12}$ against the mirror-descent step size $\lambda$ for the MDP defined in Appendix \ref{sec:implementation_details} with exact critics, evaluated at four actor budgets $n$ (light to dark). Solid lines use the mirror-descent target \eqref{eq:mirror_descent} and dashed lines the Gibbs target \eqref{eq:classic_target} at the same budget. For the Gibbs target $\lambda$ does not enter \eqref{eq:classic_target} itself, only the replay dynamics of Definition \ref{def:ema_replay}. See Appendix \ref{sec:implementation_details} for more details.}
\label{fig:actor_tracking_lambda}
\end{figure}

\section{Conclusion and limitations}
In this work, we prove convergence guarantees for a mathematical formalisation of Soft Actor Critic and demonstrate theoretical advantages of having a target policy arising from mirror descent rather than the classical Gibbs measure. 

While the results in this work can be directly demonstrated in simplified settings (see Figures \ref{fig:critic_curvature_performance} and \ref{fig:actor_tracking_lambda}), the analysis is only performed in the population setting and thus the effect of sampling remains open. Moreover, we do not incorporate explicit critic dynamics in this work which remains a challenging direction for future work.

\bibliographystyle{abbrvnat}
\bibliography{refs}

\appendix
\section{Notation}\label{sec:notation}
Let $(E,d_E)$ denote a Polish space, that is, a complete separable metric space. We always equip a Polish space with its Borel $\sigma$-field $\mathcal B(E)$. We denote by $B_b(E)$ the space of bounded measurable functions $f:E\to\mathbb R$, endowed with the supremum norm
\begin{equation}
|f|_{B_b(E)}:=\sup_{x\in E}|f(x)|.
\end{equation}
We denote by $\mathcal M(E)$ the space of finite signed measures on $E$, endowed with the total variation norm
\begin{equation}
|\mu|_{\mathcal M(E)}
:=
\sup_{|f|_{B_b(E)}\leq1}
\left|\int_E f(x)\mu(dx)\right|,
\end{equation}
and by $\mathcal P(E)\subset\mathcal M(E)$ the set of probability measures on $E$. For Polish spaces $E$ and $F$, $\mathcal P(E|F)$ denotes the set of probability kernels from $F$ to $E$. For $\mu,\nu\in\mathcal P(E)$, their relative entropy is
\begin{equation}
\KL(\mu|\nu)
:=
\int_E\log\frac{d\mu}{d\nu}(x)\mu(dx)
\end{equation}
when $\mu$ is absolutely continuous with respect to $\nu$, and $\KL(\mu|\nu):=+\infty$ otherwise.
For $\mu\in\mathcal P(E)$ and a fixed $\sigma$-finite reference measure $\nu$ on $E$, the entropy is $\mathrm{H}(\mu):=-\int_E\log\frac{d\mu}{d\nu}(x)\mu(dx)$ when $\mu$ is absolutely continuous with respect to $\nu$ and the integral is well defined, and $-\infty$ when it is not absolutely continuous; the reference measure is $\Lambda$ on $A$ and $\Lambda_{\mathbb R^M}$ for the Gaussian laws on $\mathbb R^M$.
Given $p\geq1$, a measure $\mu\in\mathcal P(E)$, and a measurable function $f:E\to\mathbb R$, we write
\begin{equation}
|f|_{L^p(\mu)}
:=
\left(\int_E|f(x)|^p\mu(dx)\right)^{\frac{1}{p}}.
\end{equation}

We let $\mathbb N:=\{0,1,2,\ldots\}$. For any $k\in\mathbb N$ with $k\geq1$, we denote the Euclidean inner product on $\mathbb R^k$ by $\langle\cdot,\cdot\rangle$ with norm denoted by $|\cdot|_2$, and we write $|z|_1:=\sum_{i=1}^k|z_i|$ for the $\ell^1$ norm. We denote the Euclidean ball of radius $R > 0$ in $\mathbb{R}^d$ by $\mathfrak B(\mathrm R):=\{w\in\mathbb R^d:|w|_2\leq\mathrm R\}$.
We denote by $\mathbb S^k$ the space of real symmetric $k\times k$ matrices and by $I_k$ the $k\times k$ identity matrix. For a real matrix $H$, its Euclidean operator norm is
\begin{equation}
|H|_{\op}:=\sup_{|v|_2=1}|Hv|_2.
\end{equation}
For $H\in\mathbb S^k$, $\lambda_{\min}(H)$ denotes its minimum eigenvalue. Given $H,G\in\mathbb S^k$, we write $H\preceq G$ if $G-H$ is positive semidefinite and $H\prec G$ if $G-H$ is positive definite. For $z=(z_1,\ldots,z_k)\in\mathbb R^k$, $\operatorname{diag}(z_1,\ldots,z_k)$ denotes the $k\times k$ diagonal matrix with diagonal entries $z_1,\ldots,z_k$.

For an open set $U\subset\mathbb R^k$, $C^2(U)$ denotes the space of twice continuously differentiable real-valued functions on $U$. For a Polish space $E$, $C(E;\mathbb S^k)$ denotes the space of continuous functions from $E$ to $\mathbb S^k$. 

To ease notation, for each $\pi \in \mathcal{P}(A|S)$, $s \in S$ and $a \in A$ we define $P_{\pi}(ds'|s) := \int_{A} P(ds'|s,a)\pi(da|s)$ and $P^{\pi}(ds',da'|s,a) := P(ds'|s,a)\pi(da'|s').$ 

\section{Background}\label{sec:background}

The state-occupancy kernel $d^{\pi} \in \mathcal{P}(S|S)$ is defined by
\begin{equation}
\label{eq:occupancy_s}
d^{\pi}(ds'|s)=(1-\gamma)\sum_{n=0}^{\infty}\gamma^nP^n_{\pi}(ds'|s)\,,
\end{equation}
where $P^n_{\pi}$ is the $n$-times product of the kernel $P_{\pi}$ with $P^0_{\pi}(ds'|s)\coloneqq \delta_s(ds')$. Given any state distribution $\rho \in \mathcal{P}(S)$, we define the state-occupancy measure as
\begin{equation}\label{eq:occupancy_defn}
d^{\pi}_{\rho}(ds) = \int_{S} d^{\pi}(ds|s')\rho(ds').
\end{equation}
\begin{definition}
\label{def:ema_replay}
Let $\rho\in\mathcal P(S)$ be an initial-state distribution, let $\chi>0$, and suppose that $0<\chi\lambda\leq 1$. Given a sequence of policies $\{\pi_{w^n}\}_{n\geq0}$, define the sequence of replay distributions $\{\rho_n\}_{n\geq0}\subset\mathcal P(S)$ recursively by
\begin{equation}
\label{eq:ema_replay}
\rho_0:=\rho,
\qquad
\rho_{n+1}
:=
(1-\chi\lambda)\rho_n
+
\chi\lambda\,d_\rho^{\pi_{w^n}},
\qquad n\geq0.
\end{equation}
Equivalently, for every $n\geq1$,
\begin{equation}
\label{eq:ema_replay_explicit}
\rho_n
=
(1-\chi\lambda)^n\rho
+
\chi\lambda
\sum_{j=0}^{n-1}
(1-\chi\lambda)^{n-1-j}
d_\rho^{\pi_{w^j}}.
\end{equation}
\end{definition}

\begin{lemma}\label{lem:replay_measure_comparison}
For every $n\geq0$, the replay measure satisfies
\begin{equation}
    \rho_n\geq(1-\gamma)\rho,
\end{equation}
\begin{equation}
    \lambda_{\min}\left(\int_Sx(s)x(s)^\top\rho_n(ds)\right)
    \geq(1-\gamma)\lambda_P.
\end{equation}
\end{lemma}
\begin{proof}[Proof of Lemma~\ref{lem:replay_measure_comparison}]
Recall that $    \rho_{n+1}:=(1-\chi\lambda)\rho_n+\chi\lambda d_\rho^{\pi_{w^n}}$. Let $\beta:=1-\chi\lambda$. Unrolling this gives
\begin{equation}\label{eq:ema_replay_unrolled}
    \rho_n=\beta^n\rho+(1-\beta)\sum_{j=0}^{n-1}\beta^{n-1-j}d_\rho^{\pi_{w^j}}.
\end{equation}
Since $d_\rho^\pi\geq(1-\gamma)\rho$ for every policy $\pi$,
\begin{equation}
    \rho_n\geq\bigl(\beta^n+(1-\gamma)(1-\beta^n)\bigr)\rho\geq(1-\gamma)\rho.
\end{equation}
Finally, for any $v\in\mathbb R^d$,
\begin{equation}
    v^\top\left(\int_Sx(s)x(s)^\top\rho_n(ds)\right)v
    \geq(1-\gamma)v^\top\left(\int_Sx(s)x(s)^\top\rho(ds)\right)v,
\end{equation}
which concludes the proof.
\end{proof}

We recall the dynamic programming principle from \citep[Appendix~B]{david_fisher}.
\begin{theorem}[Dynamic Programming Principle]
\label{thm:dynamics_programming}
Let $\tau > 0$. The optimal value function $V^{*}_{\tau}$ is the unique bounded solution of the following Bellman equation:
\[
V^{\ast}_{\tau}(s)=-\tau\log\int_{A}\exp\left(-
\frac{1}{\tau}Q^{\ast}_{\tau}(s,a)\right)\Lambda(da),
\]
where $Q^*_{\tau}\in B_b(S\times A)$ is defined by
\[
Q^{*}_{\tau}(s,a)=c(s,a)+\gamma\int_S V_{\tau}^{*}(s')P(ds'|s,a)\,,
\quad \forall (s,a)\in S\times A\,.
\]Moreover, there is an optimal policy $\pi^* \in \mathcal{P}(A|S)$  given by
\begin{equation}
\label{eq:optimal_policy}
\pi^*(da|s) = \exp\left(-\frac{1}{\tau }(Q^{\ast}_{\tau}(s,a)-V^{\ast}_{\tau}(s))\right)\Lambda(da)\,,
\quad \forall s\in S.
\end{equation}
 Finally, for every $\pi\in\mathcal P(A|S)$ with
  $\mathrm H(\pi(\cdot|\cdot))\in B_b(S)$,
  the value function $V^\pi_\tau$ is the unique bounded
  solution of the following Bellman equation for all $s\in S$
\[
V^{\pi}_{\tau}(s)=\int_{A}\left(Q_\tau^\pi(s,a)+\tau \log \frac{d \pi}{d\Lambda}(s,a)\right)\pi(da|s)\,.
\]
\end{theorem}

We now prove a performance difference lemma for policies with possibly unbounded log densities but with bounded entropy.
\begin{lemma}[Performance difference]
\label{lem:performance_diff}
For every $\rho\in\mathcal P(S)$ and every $\pi\in\mathcal{P}(A|S)$ such that $V_\tau^\pi\in B_b(S)$ and $\mathrm{H}(\pi(\cdot|\cdot))\in B_b(S)$, it holds that
\begin{align}
&V^{\pi}_\tau(\rho)-V^{\pi^*}_\tau(\rho)
\\
&\quad=
\frac{1}{1-\gamma}\int_S
\Bigg(
\int_A
Q^{\pi}_\tau(s,a)(\pi-\pi^*)(da|s)
+\tau\left(\mathrm{H}(\pi^*(\cdot|s))-\mathrm{H}(\pi(\cdot|s))\right)
\Bigg)d_\rho^{\pi^*}(ds).
\label{eq:performance_difference}
\end{align}
\end{lemma}
\begin{proof}
Fix $\rho\in\mathcal P(S)$ and $\pi \in \mathcal{P}(A|S)$. By Theorem \ref{thm:dynamics_programming}, for all $s\in S$ it holds that
\begin{align}
&\int_AQ^\pi_\tau(s,a)(\pi-\pi^*)(da|s)+\tau\left(\mathrm H(\pi^*(\cdot|s))-\mathrm H(\pi(\cdot|s))\right)\\
&\quad=\left(\int_AQ^\pi_\tau(s,a)\pi(da|s)-\tau\mathrm H(\pi(\cdot|s))\right)
-\left(\int_AQ^\pi_\tau(s,a)\pi^*(da|s)-\tau\mathrm H(\pi^*(\cdot|s))\right)\\
&\quad=V^\pi_\tau(s)-V^{\pi^*}_\tau(s)-\int_A\left(Q^\pi_\tau-Q^{\pi^*}_\tau\right)(s,a)\pi^*(da|s)\\
&\quad=V^\pi_\tau(s)-V^{\pi^*}_\tau(s)-\gamma\int_S\left(V^\pi_\tau(s')-V^{\pi^*}_\tau(s')\right)P_{\pi^*}(ds'|s),
\end{align}
where in the second equality we used that $V^\pi_\tau(s)=\int_AQ^\pi_\tau(s,a)\pi(da|s)-\tau\mathrm H(\pi(\cdot|s))$. In the third equality, we used the definition of the state-action value function in \eqref{eq:Q_func}. Now set $f:=V^\pi_\tau-V^{\pi^*}_\tau$ and let $g$ denote the left hand side. By assumption we have $V^{\pi}_{\tau} \in B_{b}(S)$ and by the Bellman principle, since $c \in B_b(S \times A)$, we have $V^{\pi^*}_{\tau} \in B_b(S)$ and thus $f \in B_b(S)$. 

Similarly, since $V^{\pi}_{\tau} \in B_b(S)$ it holds that $Q^{\pi}_{\tau} \in B_b(S\times A)$ and $\mathrm{H}(\pi(\cdot|\cdot)) \in B_b(S)$. By the Bellman principle, it also holds that $\mathrm{H}(\pi^*(\cdot|\cdot)) \in B_b(S)$. Therefore, by \citet[Lemma~3.2]{david_fisher}, for all $s \in S$ it holds that
\begin{equation}
V^\pi_\tau(s)-V^{\pi^*}_\tau(s)=\frac{1}{1-\gamma}\int_Sg(s')d^{\pi^*}(ds'|s).
\end{equation}
Integrating both sides with respect to $\rho$ and using the definition of the state occupancy measure in \eqref{eq:occupancy_defn} concludes the proof.
\end{proof}


Define 
\[
    M_{\Lambda} = \left\{ m \in \mathcal{P}(A) : \log\frac{dm}{d\Lambda} \in B_b(A)\right\},
\]
and notice that this is a convex subset of $\mathcal{P}(A)$. A proof of the following classical three-point lemma can be found in \citep{Korba22}.
\setcounter{lemma}{5}
\begin{lemma}[Three point lemma/Bregman proximal inequality]\label{lemma:three_point}
Let $G : M_{\Lambda} \to \mathbb{R}$ be convex. For all $m' \in M_{\Lambda}$ let
\begin{equation}
m^* = \argmin_{m\in M_{\Lambda}}\left\{G(m)+\operatorname{KL}(m|m') \right\}.
\end{equation}
Then for all $m \in M_{\Lambda}$ we have
\begin{equation}
G(m) + \operatorname{KL}(m|m') \geq G(m^*) + \operatorname{KL}(m|m^*) + \operatorname{KL}(m^*|m')
\end{equation}
\end{lemma}

\section{Proof of Theorem \ref{thm:conv_1}}
\label{sec:proof_of_thm_1}
\begin{proof}
        For some $\lambda > 0$ with $1-\tau\lambda \in (0,1)$ and for all $n\geq 0$ and $s \in S$, recall that
    \begin{equation}
    \begin{aligned}
        \pi^{n+1}(\cdot|s)
        &=
        \argmin_{m \in \mathcal{P}(A)}
        \Bigg\{\int_{A}
        \left(
        \widehat{Q}^n(s,a)
        +
        \tau \log \frac{d\pi^n}{d\Lambda}(s,a)
        \right)
        m(da)+\frac{1}{\lambda}
        \operatorname{KL}(m|\pi^n(\cdot|s))
        \Bigg\}.
    \end{aligned}
    \end{equation}
By \eqref{eq:optimal_policy} and \eqref{eq:mirror_descent_closed_form}, $\pi^*(\cdot|s),\pi^n(\cdot|s),\pi^{n+1}(\cdot|s)\in M_{\Lambda}$, and the minimiser in the preceding display belongs to $M_{\Lambda}$. The Three Point Lemma (Lemma \ref{lemma:three_point} with $m = \pi^*(\cdot|s)$) yields
\begin{align}\label{eq:apply_three_point_lemma}
&\int_A
\left(\widehat Q^n(s,a)+\tau\log\frac{d\pi^n}{d\Lambda}(s,a)\right)
(\pi^n-\pi^*)(da|s)  \\
&\leq
\frac1\lambda\left(
\KL(\pi^*|\pi^n)(s)-\KL(\pi^*|\pi^{n+1})(s)
-\KL(\pi^{n+1}|\pi^n)(s)
\right)\\
&\quad+
\int_A
\left(\widehat Q^n(s,a)+\tau\log\frac{d\pi^n}{d\Lambda}(s,a)\right)
(\pi^n-\pi^{n+1})(da|s).
\label{eq:three_point_rearranged}
\end{align}
We will now lower bound the left hand side. Firstly recall the standard identity
\begin{equation}
\int_A\log\frac{d\pi^n}{d\Lambda}(s,a)\pi^*(da|s)
=-\mathrm{H}(\pi^*(\cdot|s))-\KL(\pi^*(\cdot|s)|\pi^n(\cdot|s)).
\end{equation}
Consequently,
\begin{align} 
    &\int_{A} \widehat{Q}^n(s,a)(\pi^n - \pi^*)(da|s) + \tau \left(\mathrm{H}(\pi^*(\cdot|s))-\mathrm{H}(\pi^n(\cdot|s)) \right) \\
    &= \int_{A}\left(\widehat{Q}^n(s,a) + \tau\log \frac{d\pi^n}{d\Lambda}(s,a) \right)(\pi^n - \pi^*)(da|s)-\tau\KL(\pi^*|\pi^n)(s)\\
    &\leq \int_{A}\left(\widehat{Q}^n(s,a) + \tau\log \frac{d\pi^n}{d\Lambda}(s,a) \right)(\pi^n - \pi^*)(da|s),
\end{align}
Therefore \eqref{eq:apply_three_point_lemma} becomes
\begin{align}\label{eq:to_simplify_threepoint}
    &\int_{A} \widehat{Q}^n(s,a)(\pi^n - \pi^*)(da|s) + \tau \left(\mathrm{H}(\pi^*(\cdot|s))-\mathrm{H}(\pi^n(\cdot|s)) \right) \\ 
    &\leq\frac1\lambda\left(
\KL(\pi^*|\pi^n)(s)-\KL(\pi^*|\pi^{n+1})(s)
-\KL(\pi^{n+1}|\pi^n)(s)
\right)\\
&\quad+
\int_A
\left(\widehat Q^n(s,a)+\tau\log\frac{d\pi^n}{d\Lambda}(s,a)\right)
(\pi^n-\pi^{n+1})(da|s).
\end{align} Now we seek to upper bound the final term on the right hand side of \eqref{eq:to_simplify_threepoint}. By Pinsker's inequality, it holds that
\begin{align}
    &\int_A
\left(\widehat Q^n(s,a)+\tau\log\frac{d\pi^n}{d\Lambda}(s,a)\right)
(\pi^n-\pi^{n+1})(da|s)\\
&\leq \left(\widehat{Q}_{\text{max}} + \tau \sup_{n \in \mathbb{N}}\left|\log \frac{d\pi^n}{d\Lambda} \right|_{B_b(S\times A)} \right) \left|\pi^{n+1}(\cdot|s) - \pi^n(\cdot|s)\right|_{\mathcal{M}(A)} \\
&\leq \left( \widehat{Q}_{\text{max}} + \tau C_{\text{log}}\right) \sqrt{2\KL(\pi^{n+1}(\cdot|s)| \pi^n(\cdot|s))}
\end{align}
where we used that $ \sup_{n \in \mathbb{N}}\left|\log \frac{d\pi^n}{d\Lambda} \right|_{B_b(S\times A)} \leq \frac{2\widehat Q_{\max}}{\tau} := C_{\text{log}}$ by \eqref{eq:uniform_target_log_density_bound}. Now, by Young's inequality in the form $ab \leq \frac{\lambda}{2}a^2 + \frac{1}{2\lambda}b^2$, with $a = \widehat{Q}_{\mathrm{max}} + \tau C_{\text{log}}$ and $b = \sqrt{2\KL(\pi^{n+1}(\cdot|s)| \pi^n(\cdot|s))}$, we obtain
\begin{align}
    &\left( \widehat{Q}_{\mathrm{max}} + \tau C_{\text{log}}\right)
    \sqrt{2\KL(\pi^{n+1}(\cdot|s)| \pi^n(\cdot|s))}
    - \frac{\KL(\pi^{n+1}(\cdot|s)|\pi^n(\cdot|s))}{\lambda}
    \\
    &\leq
    \frac{\lambda}{2}\left( \widehat{Q}_{\mathrm{max}} + \tau C_{\text{log}}\right)^2
    + \frac{1}{2\lambda}
    \left(\sqrt{2\KL(\pi^{n+1}(\cdot|s)| \pi^n(\cdot|s))}\right)^2
    - \frac{\KL(\pi^{n+1}(\cdot|s)|\pi^n(\cdot|s))}{\lambda}
    \\
    &=
    \frac{\lambda}{2}\left( \widehat{Q}_{\mathrm{max}} + \tau C_{\text{log}}\right)^2
    + \frac{\KL(\pi^{n+1}(\cdot|s)|\pi^n(\cdot|s))}{\lambda}
    - \frac{\KL(\pi^{n+1}(\cdot|s)|\pi^n(\cdot|s))}{\lambda}
    \\
    &=
    \frac{\lambda}{2}\left( \widehat{Q}_{\mathrm{max}} + \tau C_{\text{log}}\right)^2.
\end{align}Therefore \eqref{eq:to_simplify_threepoint} simplifies to
\begin{align}\label{eq:mirror_regret_statewise}
    &\int_{A} \widehat{Q}^n(s,a)(\pi^n - \pi^*)(da|s) + \tau \left(\mathrm{H}(\pi^*(\cdot|s))-\mathrm{H}(\pi^n(\cdot|s)) \right) \\ 
    &\leq\frac1\lambda\left(
\KL(\pi^*|\pi^n)(s)-\KL(\pi^*|\pi^{n+1})(s)
\right) + \frac{\lambda}{2}\left( \widehat{Q}_{\text{max}} + \tau C_{\text{log}}\right)^2.
\end{align}Therefore after integrating over $S$ with respect to $d_{\rho}^{\pi^*}\in\mathcal{P}(S)$ and summing over $k=0,...,n-1$ we arrive at
\begin{align}\label{eq:bound_after_sum}
    &\sum_{k=0}^{n-1} \left(\int_{S} \left(\int_{A} \widehat{Q}^k(s,a)(\pi^k - \pi^*)(da|s) + \tau \left(\mathrm{H}(\pi^*(\cdot|s))-\mathrm{H}(\pi^k(\cdot|s)) \right)\right)d_{\rho}^{\pi^*}(ds) \right) \\
    &\leq \frac{1}{\lambda}\int_{S}\KL(\pi^*|\pi^0)(s)d_{\rho}^{\pi^*}(ds) + \frac{n\lambda}{2}\left( \widehat{Q}_{\text{max}} + \tau C_{\text{log}}\right)^2,
\end{align}
Now we turn to the performance difference lemma (Lemma \ref{lem:performance_diff}) for between $\pi_{w^n}$ and $\pi^*$, which states that 
\begin{align}
&
\left(V_\tau^{\pi_{w^n}}(\rho)-V_\tau^{\pi^*}(\rho)\right)\\
&\quad= \frac{1}{1-\gamma}\int_S\Bigg(
\int_AQ_\tau^{\pi_{w^n}}(s,a)
(\pi_{w^n}-\pi^*)(da|s)
+\tau\left(
\mathrm{H}(\pi^*(\cdot|s))
-\mathrm{H}(\pi_{w^n}(\cdot|s))
\right)
\Bigg)d_\rho^{\pi^*}(ds).
\end{align} Adding and subtracting $\widehat{Q}^n$ and using Assumption \ref{ass:Q_oracle}, it holds that
\begin{align}
&
\left(V_\tau^{\pi_{w^n}}(\rho)-V_\tau^{\pi^*}(\rho)\right)\\
&\quad\leq \frac{1}{1-\gamma}\int_S\Bigg(
\int_A \widehat{Q}^n(s,a)
(\pi_{w^n}-\pi^*)(da|s)
+\tau\left(
\mathrm{H}(\pi^*(\cdot|s))
-\mathrm{H}(\pi_{w^n}(\cdot|s))
\right)
\Bigg)d_\rho^{\pi^*}(ds)\\
&\qquad+\frac{2\varepsilon_{\mathrm{crit}}}{1-\gamma}.
\end{align} Furthermore adding and subtracting $\pi^n$ and $\mathrm{H}(\pi^n(\cdot|s))$ we arrive at
\begin{align}
&
V_\tau^{\pi_{w^n}}(\rho)-V_\tau^{\pi^*}(\rho) \\
&\quad\leq
\frac{1}{1-\gamma}\int_S\Bigg(
\int_A\widehat Q^n(s,a)(\pi^n-\pi^*)(da|s)
+\tau\left(
\mathrm{H}(\pi^*(\cdot|s))
-\mathrm{H}(\pi^n(\cdot|s))
\right)
\Bigg)d_\rho^{\pi^*}(ds)\nonumber\\
&\qquad+
\frac{1}{1-\gamma}\int_S\Bigg(
\int_A\widehat Q^n(s,a)(\pi_{w^n}-\pi^n)(da|s)
+\tau\left(
\mathrm{H}(\pi^n(\cdot|s))
-\mathrm{H}(\pi_{w^n}(\cdot|s))
\right)
\Bigg)d_\rho^{\pi^*}(ds)\\
&\qquad+\frac{2\epscritic}{1-\gamma}.
\label{eq:main_proof_mirror_split}
\end{align}To bound the second term on the right hand side, observe that by Theorem \ref{thm:main_value_bound} and Pinsker's inequality, it holds that
\begin{align}
&\int_A\widehat Q^n(s,a)(\pi_{w^n}-\pi^n)(da|s)
+\tau\left(
\mathrm{H}(\pi^n(\cdot|s))
-\mathrm{H}(\pi_{w^n}(\cdot|s))
\right) \\
&\leq \widehat{Q}_{\text{max}}\left|\pi_{w^n}(\cdot|s) - \pi^n(\cdot|s) \right|_{\mathcal{M}(A)} + \tau C_{\mathrm{KL}} \KL(\pi_{w^n}(\cdot|s)|\pi^n(\cdot|s))^{\frac{1}{2}} 
\\
&\leq
(\sqrt2\widehat Q_{\max}+\tau C_{\mathrm{KL}})
\KL(\pi_{w^n}(\cdot|s)|\pi^n(\cdot|s))^{\frac{1}{2}}.
\label{eq:main_proof_projection_statewise}
\end{align}Therefore after integrating over $S$ with respect to $d_{\rho}^{\pi^*}$ it also holds that
\begin{align}
    &\int_{S}\left(\int_A\widehat Q^n(s,a)(\pi_{w^n}-\pi^n)(da|s)
+\tau\left(
\mathrm{H}(\pi^n(\cdot|s))
-\mathrm{H}(\pi_{w^n}(\cdot|s))\right)
\right)d_{\rho}^{\pi^*}(ds) \\
&\leq (\sqrt2\widehat Q_{\max}+\tau C_{\mathrm{KL}}) \int_{S} \KL(\pi_{w^n}(\cdot|s) | \pi^n(\cdot|s))^{\frac{1}{2}} d_{\rho}^{\pi^*}(ds).
\end{align}By Holder's inequality and assumption \ref{as:optimal_occupancy}, it holds that
\begin{align}\label{eq:change_of_measure}
    &\int_{S} \KL(\pi_{w^n}(\cdot|s) | \pi^n(\cdot|s))^{\frac{1}{2}} d_{\rho}^{\pi^*}(ds) \\
    &= \int_{S} \KL(\pi_{w^n}(\cdot|s) | \pi^n(\cdot|s))^{\frac{1}{2}} \frac{d d_{\rho}^{\pi^*}}{d\rho}(s)\rho(ds) \\
    &\leq \left(\int_{S} \KL(\pi_{w^n}(\cdot|s) | \pi^n(\cdot|s)) \rho(ds) \right)^{\frac{1}{2}}\left(\int_{S}\left(\frac{d d_{\rho}^{\pi^*}}{d\rho}(s)\right)^2\rho(ds)\right)^\frac{1}{2} \\
    &= \left|\frac{d d_{\rho}^{\pi^*}}{d\rho} \right|_{L^2(\rho)}\left(\int_{S} \KL(\pi_{w^n}(\cdot|s) | \pi^n(\cdot|s)) \rho(ds) \right)^{\frac{1}{2}} \\
    &\leq \frac{1}{\sqrt{1-\gamma}} \left|\frac{d d_{\rho}^{\pi^*}}{d\rho} \right|_{L^2(\rho)}\left(\int_{S} \KL(\pi_{w^n}(\cdot|s) | \pi^n(\cdot|s)) \rho_n(ds) \right)^{\frac{1}{2}} \\
    &= \frac{1}{\sqrt{1-\gamma}} \left|\frac{d d_{\rho}^{\pi^*}}{d\rho} \right|_{L^2(\rho)}J_n(w^n,\pi^n)^{\frac{1}{2}} \\
    &\leq \frac{A_2}{\sqrt{1-\gamma}}J_n(w^n,\pi^n)^{\frac{1}{2}}
\end{align}
Hence substituting \eqref{eq:change_of_measure} into \eqref{eq:main_proof_mirror_split} it holds that

\begin{align}
&V_\tau^{\pi_{w^n}}(\rho)-V_\tau^{\pi^*}(\rho) \\
&\quad\leq
\frac{1}{1-\gamma}\int_S\Bigg(
\int_A\widehat Q^n(s,a)(\pi^n-\pi^*)(da|s)
+\tau\left(
\mathrm{H}(\pi^*(\cdot|s))
-\mathrm{H}(\pi^n(\cdot|s))
\right)
\Bigg)d_\rho^{\pi^*}(ds)\\
&\qquad+
\frac{A_2(\sqrt2\widehat Q_{\max}+\tau C_{\mathrm{KL}})}{(1-\gamma)^{\frac{3}{2}}}J_n(w^n,\pi^n)^{\frac{1}{2}}+\frac{2\epscritic}{1-\gamma} 
\end{align}Now let $w^n_* \in \argmin_{w \in \mathfrak{B}(\mathrm{R})} J_n(w,\pi^n)$. Adding and subtracting $J_n(w^n_*, \pi^n)$ and using that $(a + b)^{\frac{1}{2}} \leq a^{\frac{1}{2}} + b^{\frac{1}{2}}$ for any non-negative $a,b$ it holds that
\begin{align}
&V_\tau^{\pi_{w^n}}(\rho)-V_\tau^{\pi^*}(\rho) \\
&\quad\leq
\frac{1}{1-\gamma}\int_S\Bigg(
\int_A\widehat Q^n(s,a)(\pi^n-\pi^*)(da|s)
+\tau\left(
\mathrm{H}(\pi^*(\cdot|s))
-\mathrm{H}(\pi^n(\cdot|s))
\right)
\Bigg)d_\rho^{\pi^*}(ds)\\
&\qquad+
\frac{A_2(\sqrt2\widehat Q_{\max}+\tau C_{\mathrm{KL}})}{(1-\gamma)^{\frac{3}{2}}}
\left(J_n(w^n,\pi^n) - J_n(w^n_*, \pi^n) + J_n(w^n_*, \pi^n)\right)^{\frac{1}{2}}+\frac{2\epscritic}{1-\gamma} \\
&\leq \frac{1}{1-\gamma}\int_S\Bigg(
\int_A\widehat Q^n(s,a)(\pi^n-\pi^*)(da|s)
+\tau\left(
\mathrm{H}(\pi^*(\cdot|s))
-\mathrm{H}(\pi^n(\cdot|s))
\right)
\Bigg)d_\rho^{\pi^*}(ds)\\
&\qquad+
\frac{A_2(\sqrt2\widehat Q_{\max}+\tau C_{\mathrm{KL}})}{(1-\gamma)^{\frac{3}{2}}}
\left(\left(J_n(w^n,\pi^n) - J_n(w^n_*, \pi^n)\right)^{\frac{1}{2}} + J_n(w^n_*, \pi^n)^{\frac{1}{2}}\right)+\frac{2\epscritic}{1-\gamma}.
\end{align}
Summing over $k=0,\ldots,n-1$ for any $n \in \mathbb{N}$ it holds that
\begin{align}
    &\sum_{k=0}^{n-1}\left( V_\tau^{\pi_{w^k}}(\rho)-V_\tau^{\pi^*}(\rho) \right)  \\
    &\leq \frac{1}{1-\gamma}\sum_{k=0}^{n-1}\int_S\Bigg(
\int_A\widehat Q^k(s,a)(\pi^k-\pi^*)(da|s)
+\tau\left(
\mathrm{H}(\pi^*(\cdot|s))
-\mathrm{H}(\pi^k(\cdot|s))
\right)
\Bigg)d_\rho^{\pi^*}(ds)\\
&\qquad+
\frac{A_2(\sqrt2\widehat Q_{\max}+\tau C_{\mathrm{KL}})}{(1-\gamma)^{\frac{3}{2}}}
\sum_{k=0}^{n-1}\left(\left(J_k(w^k,\pi^k) - J_k(w^k_*, \pi^k)\right)^{\frac{1}{2}} + J_k(w^k_*, \pi^k)^{\frac{1}{2}}\right)+\frac{2\epscritic n}{1-\gamma}.
\end{align}
Substituting in \eqref{eq:bound_after_sum}, we arrive at
\begin{align}
    &\sum_{k=0}^{n-1}\left( V_\tau^{\pi_{w^k}}(\rho)-V_\tau^{\pi^*}(\rho) \right)  \\
    &\leq \frac{1}{\lambda(1-\gamma)}\int_{S}\KL(\pi^*|\pi^0)(s)d_{\rho}^{\pi^*}(ds)
    +\frac{n\lambda}{2(1-\gamma)}\left( \widehat{Q}_{\text{max}} + \tau C_{\text{log}}\right)^2\\
    &\qquad+\frac{A_2(\sqrt2\widehat Q_{\max}+\tau C_{\mathrm{KL}})}{(1-\gamma)^{\frac{3}{2}}}
\sum_{k=0}^{n-1}\left(\left(J_k(w^k,\pi^k) - J_k(w^k_*, \pi^k)\right)^{\frac{1}{2}} + J_k(w^k_*, \pi^k)^{\frac{1}{2}}\right)+\frac{2\epscritic n}{1-\gamma}.
\end{align}
Dividing through by $n > 0$, using that the minimum is less than the average and $J_k(w_*^k,\pi^k)=\varepsilon_{\mathrm{act}}^k$, we arrive at
\begin{align}
    &\min_{0\leq r \leq n-1}V^{\pi_{w^r}}_{\tau}(\rho) - V^{\pi^*}_{\tau}(\rho)  \\
    &\leq \frac{1}{\lambda(1-\gamma)n}\int_{S}\KL(\pi^*|\pi^0)(s)d_{\rho}^{\pi^*}(ds)
    +\frac{\lambda}{2(1-\gamma)}\left( \widehat{Q}_{\text{max}} + \tau C_{\text{log}}\right)^2\\
    &\qquad+\frac{A_2(\sqrt2\widehat Q_{\max}+\tau C_{\mathrm{KL}})}{(1-\gamma)^{\frac{3}{2}}n}
\sum_{k=0}^{n-1}\left(\left(J_k(w^k,\pi^k) - J_k(w^k_*, \pi^k)\right)^{\frac{1}{2}} + (\varepsilon_{\mathrm{act}}^k)^{\frac{1}{2}}\right)+\frac{2\epscritic}{1-\gamma}\\
    &\leq \mathrm C\left(\frac{1}{\lambda n}+\lambda+
    \frac{1}{n}\sum_{k=0}^{n-1}\left(J_k(w^k,\pi^k) - J_k(w^k_*, \pi^k)\right)^{\frac{1}{2}}+\varepsilon_{\mathrm{approx}}(n)\right),
\end{align}
where
\begin{equation}
\begin{aligned}
\mathrm C:=\max\Bigg\{&
\frac{1}{1-\gamma}\int_{S}\KL(\pi^*|\pi^0)(s)d_{\rho}^{\pi^*}(ds),
\frac{(\widehat Q_{\max}+\tau C_{\text{log}})^2}{2(1-\gamma)},\\
&\frac{A_2(\sqrt2\widehat Q_{\max}+\tau C_{\mathrm{KL}})}{(1-\gamma)^{\frac{3}{2}}},
\frac{2}{1-\gamma}
\Bigg\}.
\end{aligned}
\end{equation}
\end{proof}

\section{Proof of Theorem \ref{thm:strong_convexity_L_smoothness_objective}}
\label{sec:proof_of_convexity_conditions}
\begin{proof}
For any $n\in\mathbb N$ and $w\in\mathfrak B(\mathrm R)$, recall that
\begin{equation}
    J_{n+1}(w,\pi^{n+1})
    =
    \int_S \operatorname{KL}(\pi_w(\cdot|s)|\pi^{n+1}(\cdot|s))\rho_{n+1}(ds).
\end{equation}
Focusing on the integrand, by Lemma \ref{lem:dupuis-ellis}, for every $s\in S$ it holds that
\begin{align}\label{eq:apply_dupuis}
    \operatorname{KL}(\pi_w(\cdot|s)|\pi^{n+1}(\cdot|s))
    &=\operatorname{KL}\left((\tanh)_{\#}q_w(\cdot|s)|\pi^{n+1}(\cdot|s)\right)\\
    &=\operatorname{KL}\left(q_w(\cdot|s)|\left(\tanh^{-1}\right)_{\#}\pi^{n+1}(\cdot|s)\right).
\end{align}
Let $\Lambda_{\mathbb R^M}$ denote Lebesgue measure on $\mathbb R^M$. Since $\tanh:\mathbb R^M\to A$ is a measurable bijection with measurable inverse, performing the change of variable $a=\tanh u$ with $u \in \mathbb{R}^M$, for every $B\in\mathcal B(\mathbb R^M)$ it holds that
\begin{align}
\left((\tanh^{-1})_\#\pi^{n+1}\right)(B|s)
&=\pi^{n+1}(\tanh(B)|s)\\
&=\int_{\tanh(B)}\frac{d\pi^{n+1}}{d\Lambda}(s,a)\Lambda(da)\\
&=\int_B\frac{d\pi^{n+1}}{d\Lambda}(s,\tanh u)
2^{-M}\det D(\tanh u)\Lambda_{\mathbb R^M}(du).
\end{align}
Thus it holds that
\begin{align}\label{eq:uniform_push_forward_density}
    \frac{d\left(\tanh^{-1}\right)_{\#}\pi^{n+1}}{d\Lambda_{\mathbb R^M}}(s,u)
    &=
    \frac{\det D(\tanh u)}{2^{M}}
    \frac{d\pi^{n+1}}{d\Lambda}(s,\tanh u)\\
    &=
    \frac{\det D(\tanh u)}{2^M\overline Z_{n+1}(s)}
    \exp\left(
    -\lambda\sum_{k=0}^n
    (1-\tau\lambda)^{n-k}\widehat Q^k(s,\tanh u)
    \right),
\end{align}
such that $\overline Z_{n+1}(s)
    :=
    \int_A
    \exp\left(
    -\lambda\sum_{k=0}^n
    (1-\tau\lambda)^{n-k}\widehat Q^k(s,a)
    \right)
    \Lambda(da)$. To ease notation, let
\begin{equation}
    F_n(s,u)
    :=
    \lambda\sum_{k=0}^n
    (1-\tau\lambda)^{n-k}\widehat Q^k(s,\tanh u).
\end{equation}
Starting from \eqref{eq:apply_dupuis} we therefore have that
\begin{align}
&\operatorname{KL}(\pi_w(\cdot|s)|\pi^{n+1}(\cdot|s))\\
&=\operatorname{KL}\left(q_w(\cdot|s)|\left(\tanh^{-1}\right)_{\#}\pi^{n+1}(\cdot|s)\right)\\
&=\int_{\mathbb R^M}
\log\frac{dq_w}{d\left(\tanh^{-1}\right)_{\#}\pi^{n+1}}(s,u)q_w(du|s)\\
&=\int_{\mathbb R^M}
\log\left(
\frac{dq_w}{d\Lambda_{\mathbb R^M}}(s,u)
\frac{d\Lambda_{\mathbb R^M}}{d\left(\tanh^{-1}\right)_{\#}\pi^{n+1}}(s,u)
\right)q_w(du|s)\\
&=\int_{\mathbb R^M}
\log\frac{dq_w}{d\Lambda_{\mathbb R^M}}(s,u)q_w(du|s)
-\int_{\mathbb R^M}
\log\frac{d\left(\tanh^{-1}\right)_{\#}\pi^{n+1}}{d\Lambda_{\mathbb R^M}}(s,u)
q_w(du|s)\\
&=-\mathrm{H}(q_w(\cdot|s))+ \int_{\mathbb R^M}
\left(
\log\left(2^M\overline Z_{n+1}(s)\right)
+F_n(s,u)
-\log\det D(\tanh u)
\right)q_w(du|s)\\
&=-\mathrm{H}(q_w(\cdot|s))
+\log\left(2^M\overline Z_{n+1}(s)\right)+\int_{\mathbb R^M}
\left(2\sum_{i=1}^M\log\cosh u_i+F_n(s,u)\right)q_w(du|s)\\
&=\int_{\mathbb R^M}
\left(2\sum_{i=1}^M\log\cosh u_i+F_n(s,u)\right)q_w(du|s)
+\log\frac{2^M\overline Z_{n+1}(s)}{(2\pi)^{\frac{M}{2}}\sqrt{\det\Sigma}}
-\frac M2,
\label{eq:uniform_objective_decomp}
\end{align}
where in the fifth equality we used \eqref{eq:uniform_push_forward_density} and the definition of the entropy $\mathrm{H}(q_w(\cdot|s))$. In the sixth equality we used $\det D(\tanh u)=\prod_{i=1}^M\cosh^{-2}u_i$ and in the final equality we used that $\mathrm{H}(q_w(\cdot|s))
=
\log((2\pi)^{\frac{M}{2}}\sqrt{\det\Sigma})+\frac M2$. We now further ease notation through
\begin{equation}
    \widetilde F_n(s,u):=2\sum_{i=1}^M\log\cosh u_i+F_n(s,u),
    \qquad
    C_n:=\int_S\left(\log\frac{2^M\overline Z_{n+1}(s)}{(2\pi)^{\frac{M}{2}}\sqrt{\det\Sigma}}-\frac M2\right)\rho_{n+1}(ds).
\end{equation}
Therefore integrating \eqref{eq:uniform_objective_decomp} over $S$ with respect to $\rho_{n+1}$ it holds that
\begin{equation}\label{eq:uniform_J_decomp}
    J_{n+1}(w,\pi^{n+1})
    =
    \int_S\int_{\mathbb R^M}\widetilde F_n(s,u)q_w(du|s)\rho_{n+1}(ds)
    +C_n.
\end{equation}

Now recall that for $f\in C^2(A)$ and $a=\tanh u$, the definition of $\mathcal T$ and the chain rule give, for every $i,j\in\{1,\ldots,M\}$,
\begin{align}
\partial_{u_i u_j}(f(\tanh u))
&=(1-a_i^2)(1-a_j^2)\partial_{a_i a_j}f(a)
-2\mathbf 1_{\{i=j\}}a_i(1-a_i^2)\partial_{a_i}f(a),\\
\nabla_u^2(f(\tanh u))
&=D(a)^{\frac{1}{2}}(\mathcal T f)(a)D(a)^{\frac{1}{2}}.
\end{align}
Moreover recall that $R_{\widehat Q}=\sup_{k\in\mathbb N,\ s\in S,\ a\in A}|\mathcal T(\widehat Q^k(s,\cdot))(a)|_{\op}$. Since $0\prec D(a)\preceq I_M$, it therefore holds that
\begin{align}
\nabla_u^2\widetilde F_n(s,u)
&=
D(a)^{\frac{1}{2}}
\left(
2I_M+
\lambda\sum_{k=0}^n(1-\tau\lambda)^{n-k}
\mathcal T(\widehat Q^k(s,\cdot))(a)
\right)D(a)^{\frac{1}{2}}\\
&\succeq
\left(
2-
R_{\widehat Q}\lambda
\sum_{k=0}^n(1-\tau\lambda)^{n-k}
\right)D(a).
\end{align}
Analogously it also holds that
\begin{align}
\nabla_u^2\widetilde F_n(s,u)
&=
D(a)^{\frac{1}{2}}
\left(
2I_M+
\lambda\sum_{k=0}^n(1-\tau\lambda)^{n-k}
\mathcal T(\widehat Q^k(s,\cdot))(a)
\right)D(a)^{\frac{1}{2}}\\
&\preceq
\left(
2+
R_{\widehat Q}\lambda
\sum_{k=0}^n(1-\tau\lambda)^{n-k}
\right)D(a).
\end{align}
and thus after using that $\lambda\sum_{k=0}^n(1-\tau\lambda)^{n-k}\leq\frac1\tau$, for all $s\in S$ and $u\in\mathbb R^M$ it holds that
\begin{equation}\label{eq:uniform_derivative_bounds}
\left(2-\frac{R_{\widehat Q}}{\tau}\right)D(a)
\preceq
\nabla_u^2\widetilde F_n(s,u)
\preceq
\left(2+\frac{R_{\widehat Q}}{\tau}\right)D(a).
\end{equation}
Since $a=\tanh u$, we may write these bounds as
\begin{equation}\label{eq:uniform_derivative_bounds_u}
\left(2-\frac{R_{\widehat Q}}{\tau}\right)D(\tanh u)
\preceq
\nabla_u^2\widetilde F_n(s,u)
\preceq
\left(2+\frac{R_{\widehat Q}}{\tau}\right)D(\tanh u).
\end{equation}

Now we have upper and lower bounds on the Hessian of $\widetilde F_n$, we translate these into upper and lower bounds on the Hessian of the objective $J_{n+1}(w,\pi^{n+1})
=\int_S\int_{\mathbb R^M}\widetilde F_n(s,u)q_w(du|s)\rho_{n+1}(ds)+C_n$. On that end, firstly observe that after performing the change of variable $u=m_w(s)+\Sigma^{\frac{1}{2}}z$ it holds that
\begin{equation}
\int_{\mathbb R^M}\widetilde F_n(s,u)q_w(du|s)
=
\int_{\mathbb R^M}\widetilde F_n(s,m_w(s)+\Sigma^{\frac{1}{2}}z)q^1(dz).
\end{equation}
Applying the chain rule twice (after recalling that $m_w(s) = x(s)^\top w$) yields
\begin{align}
\nabla_w^2
\left(\int_{\mathbb R^M}\widetilde F_n(s,u)q_w(du|s)\right)
&=
x(s)\left(
\int_{\mathbb R^M}
\left.\nabla_u^2\widetilde F_n(s,u)\right|_{u=m_w(s)+\Sigma^{\frac{1}{2}}z}
q^1(dz)\right)x(s)^\top.
\end{align}
Consequently, for every $v\in\mathbb R^d$, it holds that
\begin{align}\label{eq:consequently}
&v^\top\nabla_w^2
\left(\int_{\mathbb R^M}\widetilde F_n(s,u)q_w(du|s)\right)v\\
&\quad=
\int_{\mathbb R^M}
(x(s)^\top v)^\top
\left.\nabla_u^2\widetilde F_n(s,u)\right|_{u=m_w(s)+\Sigma^{\frac{1}{2}}z}
(x(s)^\top v)q^1(dz).
\end{align}
Now we apply the lower bound in \eqref{eq:uniform_derivative_bounds_u} at $u=m_w(s)+\Sigma^{\frac{1}{2}}z$ to arrive at
\begin{align}\label{eq:use_tanh_bound}
&v^\top\nabla_w^2
\left(\int_{\mathbb R^M}\widetilde F_n(s,u)q_w(du|s)\right)v\\
&\quad\geq
\left(2-\frac{R_{\widehat Q}}{\tau}\right)
\sum_{i=1}^M(x(s)^\top v)_i^2
\int_{\mathbb R^M}
\left(1-\tanh^2(m_w(s)_i+\sigma_i z_i)\right)q^1(dz).
\end{align}
It thus remains to lower bound each integral
$\int_{\mathbb R^M}(1-\tanh^2(m_w(s)_i+\sigma_i z_i))q^1(dz)$.
To that end, since the integrand is non-negative and $r\mapsto1-\tanh^2r$ is even and decreasing for $r\geq0$, for every $i\in\{1,\ldots,M\}$, $w\in\mathfrak B(\mathrm R)$ and $s\in S$, it holds that
\begin{align}
\int_{\mathbb R^M}\left(1-\tanh^2(m_w(s)_i+\sigma_i z_i)\right)q^1(dz)
&\geq
\int_{\{z\in\mathbb R^M:|z_i|\leq1\}}
\left(1-\tanh^2(m_w(s)_i+\sigma_i z_i)\right)q^1(dz)\\
&\geq
\left(1-\tanh^2(\mathrm R+\sigma_i)\right)
\int_{\{z\in\mathbb R^M:|z_i|\leq1\}}q^1(dz)\\
&=
\left(1-\tanh^2(\mathrm R+\sigma_i)\right)\mathbb P(|Z_i|\leq1)\\
&=
\left(1-\tanh^2(\mathrm R+\sigma_i)\right)(2\Phi(1)-1).
\end{align}
Here $Z\sim q^1$, $\Phi$ is the standard univariate Gaussian distribution function, and $|m_w(s)_i|\leq|x(s)^\top w|_2\leq\mathrm R$. Taking the minimum of these lower bounds over $i$ gives
\begin{align}
\int_{\mathbb R^M}\left(1-\tanh^2(m_w(s)_i+\sigma_i z_i)\right)q^1(dz)
&\geq c_{\Sigma,\mathrm R}\\
&=\min_{1\leq j\leq M}\left(1-\tanh^2(\mathrm R+\sigma_j)\right)(2\Phi(1)-1)>0.
\end{align}
Let $c_{\Sigma,\mathrm{R}} := \min_{1\leq j\leq M}\left(1-\tanh^2(\mathrm R+\sigma_j)\right)(2\Phi(1)-1)$. \eqref{eq:use_tanh_bound} becomes
\begin{equation}\label{eq:uniform_state_hessian_lower}
v^\top\nabla_w^2
\left(\int_{\mathbb R^M}\widetilde F_n(s,u)q_w(du|s)\right)v
\geq
\left(2-\frac{R_{\widehat Q}}{\tau}\right)c_{\Sigma,\mathrm R}|x(s)^\top v|_2^2.
\end{equation}
Integrating \eqref{eq:uniform_state_hessian_lower} over $S$ with respect to $\rho_{n+1}$ gives
\begin{align}\label{eq:final_calc_for_sc}
v^\top\nabla_w^2J_{n+1}(w,\pi^{n+1})v
&\geq
\left(2-\frac{R_{\widehat Q}}{\tau}\right)c_{\Sigma,\mathrm R}
\int_S|x(s)^\top v|_2^2\rho_{n+1}(ds)\\
&=
\left(2-\frac{R_{\widehat Q}}{\tau}\right)c_{\Sigma,\mathrm R}
v^\top\left(\int_Sx(s)x(s)^\top\rho_{n+1}(ds)\right)v\\
&\geq
(1-\gamma)\left(2-\frac{R_{\widehat Q}}{\tau}\right)c_{\Sigma,\mathrm R}
v^\top\left(\int_Sx(s)x(s)^\top\rho(ds)\right)v\\
&\geq
(1-\gamma)\left(2-\frac{R_{\widehat Q}}{\tau}\right)c_{\Sigma,\mathrm R}\lambda_P|v|_2^2,
\end{align}
where we used Lemma \ref{lem:replay_measure_comparison} in the second inequality and Assumption \ref{as:policy_Evalues} in the final inequality. Thus \eqref{eq:final_calc_for_sc} implies that the Hessian of $w\mapsto J_{n+1}(w,\pi^{n+1})$ has a minimum eigenvalue bounded from below by $(1-\gamma)\lambda_P(2-\frac{R_{\widehat Q}}{\tau})c_{\Sigma,\mathrm R}$ for all $n\in\mathbb N$, which in turn implies that $w\mapsto J_{n+1}(w,\pi^{n+1})$ is $(1-\gamma)\lambda_P(2-\frac{R_{\widehat Q}}{\tau})c_{\Sigma,\mathrm R}$-strongly convex for all $n\in\mathbb N$ and $w\in\mathfrak B(\mathrm R)$.

Now we turn to analogously demonstrate smoothness. Using the upper bound in \eqref{eq:uniform_derivative_bounds_u} in \eqref{eq:consequently}, we have
\begin{align}
&v^\top\nabla_w^2
\left(\int_{\mathbb R^M}\widetilde F_n(s,u)q_w(du|s)\right)v\\
&\quad\leq
\left(2+\frac{R_{\widehat Q}}{\tau}\right)
\sum_{i=1}^M(x(s)^\top v)_i^2
\int_{\mathbb R^M}
\left(1-\tanh^2(m_w(s)_i+\sigma_i z_i)\right)q^1(dz).
\end{align}
Since $1-\tanh^2u_i\leq1$ and $q^1$ is a probability measure, it holds that
$\int_{\mathbb R^M}(1-\tanh^2(m_w(s)_i+\sigma_i z_i))q^1(dz)\leq1$
for every $i$. Hence it holds that
\begin{equation}\label{eq:uniform_state_hessian_upper}
v^\top\nabla_w^2
\left(\int_{\mathbb R^M}\widetilde F_n(s,u)q_w(du|s)\right)v
\leq
\left(2+\frac{R_{\widehat Q}}{\tau}\right)|x(s)^\top v|_2^2.
\end{equation}
Integrating \eqref{eq:uniform_state_hessian_upper} over $S$ with respect to $\rho_{n+1}$ gives
\begin{align}
v^\top\nabla_w^2J_{n+1}(w,\pi^{n+1})v
&\leq
\left(2+\frac{R_{\widehat Q}}{\tau}\right)
\int_S|x(s)^\top v|_2^2\rho_{n+1}(ds)\\
&\leq
\left(2+\frac{R_{\widehat Q}}{\tau}\right)|v|_2^2,
\end{align}
where the final inequality follows from $|x(s)|_{\op}\leq1$ for all $s\in S$. Thus, the maximum eigenvalue of the Hessian of $w\mapsto J_{n+1}(w,\pi^{n+1})$ is bounded from above by $2+\frac{R_{\widehat Q}}{\tau}$ uniformly over $n\in\mathbb N$ and $w\in\mathfrak B(\mathrm R)$. Hence, $w\mapsto J_{n+1}(w,\pi^{n+1})$ is $2+\frac{R_{\widehat Q}}{\tau}$-smooth.
For the Gibbs target \eqref{eq:classic_target}, the same calculations hold with $F_n(s,u)=\frac{1}{\tau}\widehat Q^n(s,\tanh u)$.
Since $|\mathcal T(\widehat Q^n(s,\cdot))(a)|_{\op}\leq R_{\widehat Q}$, the bound \eqref{eq:uniform_derivative_bounds_u} still holds, and the remaining argument gives the same strong-convexity and smoothness constants for $w\mapsto J_{n+1}(w,\pi^{n+1}_{G})$.
\end{proof}

\begin{remark}[Linear Fourier critic]\label{rem:linear_fourier_critic}
Suppose that $\widehat Q^n(s,a)=(\theta^n)^\top\xi(s,a)$, where $|\theta^n|_2\leq\mathrm C$ and $\xi=(\xi_1,\ldots,\xi_K)$ with $\xi_j(s,a)=\phi_j(s)\cos(\omega_j^\top a)$ with $\omega_j\in\mathbb R^M$ and $|\phi_j(s)|\leq1$. A direct calculation gives that for each $s \in S$ and $a \in A$, $|\nabla_a\xi_j(s,a)|_2\leq|\omega_j|_2$ and $|\nabla_a^2\xi_j(s,a)|_{\op}\leq|\omega_j|_2^2$. Therefore it also holds that
\begin{align}
|\mathcal T\xi_j(s,\cdot)(a)|_{\op}
&\leq
|D(a)^{\frac{1}{2}}|_{\op}^2|\nabla_a^2\xi_j(s,a)|_{\op}
+2\max_{1\leq i\leq M}|a_i\partial_{a_i}\xi_j(s,a)|\\
&\leq|\omega_j|_2^2+2|\omega_j|_2.
\end{align}
Hence we have
  \begin{align}
  |\mathcal T\widehat Q^n(s,\cdot)(a)|_{\op}
  &\leq
  \sum_{j=1}^K
  |\theta_j^n|
  |\mathcal T\xi_j(s,\cdot)(a)|_{\op}\leq
  \mathrm C
  \left(
  \sum_{j=1}^K
  \left(
  |\omega_j|_2^2+2|\omega_j|_2
  \right)^2
  \right)^{\frac12}.
  \end{align}
Thus the conditions of Theorem \ref{thm:strong_convexity_L_smoothness_objective} is satisfied whenever
\begin{equation}
\left(\sum_{j=1}^K(|\omega_j|_2^2+2|\omega_j|_2)^2\right)^{\frac12}<\frac{2\tau}{\mathrm C}.
\end{equation}
\end{remark}

\section{Proof of Lemma \ref{lemma:policy_drift}}
\setcounter{lemma}{8}
\begin{lemma}\label{lemma:policy_drift}
  Let Assumptions \ref{ass:Q_oracle} and \ref{as:policy_Evalues} and the condition of Theorem \ref{thm:strong_convexity_L_smoothness_objective} hold and let $0<\chi\lambda\leq1$.
  Then for all $n \in \mathbb{N}$ it holds that
  \begin{align}
      |w_*^{n+1}-w_*^n|_2^2
      &\leq C_{\text{drift}}\lambda
  \end{align}
  such that $C_{\text{drift}} = \frac{1}{\kappa}\left(16\widehat Q_{\max}+4\chi \left(K_{\Sigma,\mathrm R}+\frac{2\widehat Q_{\max}}{\tau}\right)\right)$ and $\kappa:=
      (1-\gamma)\left(
      2
      -
      \frac{R_{\widehat Q}}{\tau}
      \right)c_{\Sigma,\mathrm R}\lambda_P$.
  \end{lemma}

  \begin{proof}
  Firstly define $D_n(w):=J_{n+1}(w,\pi^{n+1})-J_n(w,\pi^n)$ for any $w \in \mathfrak{B}(\mathrm{R})$ and observe that by Theorem \ref{thm:strong_convexity_L_smoothness_objective}, $w \mapsto J_n(w,\pi^n)$ is $\kappa$-strongly convex for all $w \in \mathfrak{B}(R)$ and $n \in \mathbb{N}$, which implies that
    \begin{equation}\label{eq:quadratic_growth_actor_minimiser}
      \frac{\kappa}{2}|w_*^{n+1}-w_*^n|_2^2
      \leq
      J_{n+1}(w_*^n,\pi^{n+1})
      -
      J_{n+1}(w_*^{n+1},\pi^{n+1}).
  \end{equation}
  Moreover since $J_{n+1}(w_*^n,\pi^{n+1})
      -J_{n+1}(w_*^{n+1},\pi^{n+1})
      \leq D_n(w_*^n)-D_n(w_*^{n+1})$ for any $ n \in \mathbb{N}$, it suffices to bound $D_n(w_*^n)-D_n(w_*^{n+1})$. To that end define
  \begin{align}
      l_n(s,a)
      &:=
      \log\frac{d\pi^n}{d\Lambda}(s,a)
      -
      \int_A
      \log\frac{d\pi^n}{d\Lambda}(s,a')\Lambda(da')
  \end{align}
  Observe that by \eqref{eq:mirror_descent_closed_form}, a direct calculation yields
  \begin{align}
      l_{n+1}(s,a)
      &=
      (1-\tau\lambda)l_n(s,a)
      -
      \lambda
      \left(
      \widehat Q^n(s,a)
      -
      \int_A \widehat Q^n(s,a')\pi^0(da'|s)
      \right).
  \end{align}
  Subtracting $l_n$ from both sides for each $s \in S$ and $a \in A$ gives
  \begin{align}
      l_{n+1}(s,a)-l_n(s,a)
      &=
      -\tau\lambda l_n(s,a)
      -
      \lambda
      \left(
      \widehat Q^n(s,a)
      -
      \int_A \widehat Q^n(s,a')\pi^0(da'|s)
      \right).
  \end{align}
  Taking absolute values and using the triangle inequality, we obtain
  \begin{align}\label{eq:l_n_difference_abs}
      |l_{n+1}(s,a)-l_n(s,a)|
      &\leq
      \tau\lambda|l_n(s,a)|+
      \lambda
      \left|
      \widehat Q^n(s,a)
      -
      \int_A \widehat Q^n(s,a')\pi^0(da'|s)
      \right|.
  \end{align}
  Moreover observe that
  \begin{align}
      \left|
      \widehat Q^n(s,a)
      -
      \int_A \widehat Q^n(s,a')\pi^0(da'|s)
      \right|
      &\leq
      |\widehat Q^n(s,a)|
      +
      \int_A |\widehat Q^n(s,a')|\pi^0(da'|s)\\
      &\leq
      2|\widehat Q^n|_{B_b(S\times A)}\\
      &\leq
      2\widehat Q_{\max},
  \end{align}
  Thus, for all $n \in \mathbb{N}$, \eqref{eq:l_n_difference_abs} becomes
  \begin{align}\label{eq:l_n_difference}
      \left|l_{n+1}-l_n\right|_{B_b(S\times A)}
      &\leq
      \lambda
      \left(
      \tau|l_n|_{B_b(S\times A)}
      +
      2\widehat Q_{\max}
      \right).
  \end{align}
  By \eqref{eq:mirror_descent_closed_form}, $\Lambda(A)=1$ and the triangle inequality, it holds that
  \begin{align}
      |l_n|_{B_b(S\times A)}
      &\leq2\widehat Q_{\max}\lambda\sum_{k=1}^{n}(1-\tau\lambda)^{n-k}
      \leq \frac{2\widehat Q_{\max}}{\tau}.
  \end{align}
  Substituting this into \eqref{eq:l_n_difference}, we obtain
  \begin{align}\label{eq:ln_diff_bound}
      \left|l_{n+1}-l_n\right|_{B_b(S\times A)}
      &\leq
      \lambda
      \left(
      \tau
      \left(
          \frac{2\widehat Q_{\max}}{\tau}
      \right)
      +
      2\widehat Q_{\max}
      \right)\\
      &=
      4\lambda
      \widehat Q_{\max}.
  \end{align}

  Thus for every $w,w'\in\mathfrak{B}(\mathrm R)$ and $n \in \mathbb{N}$, by \eqref{eq:ln_diff_bound} it holds that
  \begin{align}\label{eq:target_objective_drift_uniform}
      &\left|
      \left(J_{n+1}(w,\pi^{n+1})-J_{n+1}(w,\pi^{n})\right)
      -
      \left(J_{n+1}(w',\pi^{n+1})-J_{n+1}(w',\pi^{n})\right)
      \right|\\
      &=
      \left|
      \int_S\int_A
      \left(l_n(s,a)-l_{n+1}(s,a)\right)
      (\pi_w-\pi_{w'})(da|s)\rho_{n+1}(ds)
      \right|\\
      &\leq
      2\left|l_{n+1}-l_n\right|_{B_b(S\times A)} \\
      &\leq 8\lambda\widehat Q_{\max}.
  \end{align}
Now recall that by the proof of Theorem \ref{thm:main_value_bound}, for all $n \in \mathbb{N}$ it holds that $\operatorname{KL}(\pi_w(\cdot|s)|\pi^n(\cdot|s))
  \leq
  K_{\Sigma,\mathrm R}
  +\frac{2\widehat Q_{\max}}{\tau}
  =:M_J$ for all $s \in S$. Moreover, by definition of the replay distribution we have $\rho_{n+1}-\rho_n=\chi\lambda(d_\rho^{\pi_{w^n}}-\rho_n)$. Thus it follows that
    \begin{align}\label{eq:replay_objective_drift_uniform}
  &\left|
  \left(J_{n+1}(w,\pi^n)-J_n(w,\pi^n)\right)
  -
  \left(J_{n+1}(w',\pi^n)-J_n(w',\pi^n)\right)
  \right|
  \\
  &=
  \left|
  \int_S
  \operatorname{KL}(\pi_w(\cdot|s)|\pi^n(\cdot|s))
  (\rho_{n+1}-\rho_n)(ds)
  -
  \int_S
  \operatorname{KL}(\pi_{w'}(\cdot|s)|\pi^n(\cdot|s))
  (\rho_{n+1}-\rho_n)(ds)
  \right|
  \\
  &=
  \chi\lambda
  \left|
  \int_S
  \left(
  \operatorname{KL}(\pi_w(\cdot|s)|\pi^n(\cdot|s))
  -
  \operatorname{KL}(\pi_{w'}(\cdot|s)|\pi^n(\cdot|s))
  \right)
  d_\rho^{\pi_{w^n}}(ds)
  \right.
  \nonumber\\
  &\qquad\left.
  -
  \int_S
  \left(
  \operatorname{KL}(\pi_w(\cdot|s)|\pi^n(\cdot|s))
  -
  \operatorname{KL}(\pi_{w'}(\cdot|s)|\pi^n(\cdot|s))
  \right)
  \rho_n(ds)
  \right| \\
  &\leq
  2\chi\lambda M_J.
  \end{align}
  Now define $D_n(w):=J_{n+1}(w,\pi^{n+1})-J_n(w,\pi^n)$. Combining
  \eqref{eq:target_objective_drift_uniform} and
  \eqref{eq:replay_objective_drift_uniform} yields
  \begin{equation}\label{eq:objective_drift_uniform}
      |D_n(w)-D_n(w')|
      \leq
      \lambda\left(8\widehat Q_{\max}+2\chi M_J\right).
  \end{equation}

  Now observe that, since $w\mapsto J_{n+1}(w,\pi^{n+1})$ is
  $\kappa$-strongly convex on $\mathfrak{B}(\mathrm R)$ and $w_*^{n+1}$ minimises
  $w\mapsto J_{n+1}(w,\pi^{n+1})$ over $\mathfrak{B}(\mathrm R)$, we have
  \begin{equation}
      \frac{\kappa}{2}|w_*^{n+1}-w_*^n|_2^2
      \leq
      J_{n+1}(w_*^n,\pi^{n+1})
      -
      J_{n+1}(w_*^{n+1},\pi^{n+1}).
  \end{equation}
  Adding and subtracting $J_n(w_*^n,\pi^n)$ and $J_n(w_*^{n+1},\pi^n)$ and using \eqref{eq:objective_drift_uniform} yields
  \begin{align}\label{eq:actor_minimiser_drift_before_uniform_bound}
      &J_{n+1}(w_*^n,\pi^{n+1})
      -J_{n+1}(w_*^{n+1},\pi^{n+1})
      \leq D_n(w_*^n)-D_n(w_*^{n+1})
      \leq \lambda\left(8\widehat Q_{\max}+2\chi M_J\right).
  \end{align}
  Combining \eqref{eq:quadratic_growth_actor_minimiser} and
  \eqref{eq:actor_minimiser_drift_before_uniform_bound}, and rearranging,
  concludes the proof. 
  \end{proof}

\section{Proof of Lemma \ref{lemma:actor_tracking}}
\label{sec:proof_of_actor_tracking}
\begin{proof}
Let $\kappa:=(1-\gamma)\left(
      2
      -
      \frac{R_{\widehat Q}}{\tau}
      \right)c_{\Sigma,\mathrm R}\lambda_P$ and $
      L
      :=
      2
      +
      \frac{R_{\widehat Q}}{\tau}$. By assumption and Theorem \ref{thm:strong_convexity_L_smoothness_objective}, $w \mapsto J_n(w,\pi^n)$ is $\kappa$-strongly convex and $L$-smooth for any $n \in \mathbb{N}$. Now recall that by definition, for any $n \geq 1$,
    \begin{equation}
        w^n
        =
        \mathcal P_{\mathfrak{B}(\mathrm R)}
        \left(
        w^{n-1}
        -
        h\nabla_w J_n(w^{n-1},\pi^n)
        \right).
    \end{equation}
    Since $w_*^n =  \argmin_{w \in \mathfrak{B}(\mathrm R)} J_n(w,\pi^n)$, the first-order optimality condition gives that for all $w\in\mathfrak{B}(\mathrm R)$
    \begin{equation}\label{eq:actor_projected_optimality}
        \left(\nabla_w J_n(w_*^n,\pi^n)\right)^{\top}
        (w-w_*^n)
        \geq
        0.
    \end{equation}
    Since $\mathfrak{B}(\mathrm R)$ is closed and convex, the projection characterisation gives $x=\mathcal P_{\mathfrak{B}(\mathrm R)}(y)$ if and only if $(y-x)^{\top}(w-x)\leq0$ for all $w\in\mathfrak{B}(\mathrm R)$.
    Applying this with $x=w_*^n$ and $y=w_*^n-h\nabla_w J_n(w_*^n,\pi^n)$,
  \eqref{eq:actor_projected_optimality} implies that
    \begin{equation}
        w_*^n
        =
        \mathcal P_{\mathfrak{B}(\mathrm R)}
        \left(
        w_*^n
        -
        h\nabla_w J_n(w_*^n,\pi^n)
        \right).
    \end{equation}
    Therefore, it holds that
    \begin{align}\label{eq:first_non_expansive_calc}
        |w^n-w_*^n|_2^2
        &=
        \left|
        \mathcal P_{\mathfrak{B}(\mathrm R)}
        \left(
        w^{n-1}
        -
        h\nabla_w J_n(w^{n-1},\pi^n)
        \right)
        -
        w_*^n
        \right|_2^2\\
        &=
        \left|
        \mathcal P_{\mathfrak{B}(\mathrm R)}
        \left(
        w^{n-1}
        -
        h\nabla_w J_n(w^{n-1},\pi^n)
        \right)
        -
        \mathcal P_{\mathfrak{B}(\mathrm R)}
        \left(
        w_*^n
        -
        h\nabla_w J_n(w_*^n,\pi^n)
        \right)
        \right|_2^2\\
        &\leq
        \left|
        w^{n-1}-w_*^n
        -
        h
        \left(
        \nabla_w J_n(w^{n-1},\pi^n)
        -
        \nabla_w J_n(w_*^n,\pi^n)
        \right)
        \right|_2^2\\
        &=
        |w^{n-1}-w_*^n|_2^2
        -
        2h
        \left(
        \nabla_w J_n(w^{n-1},\pi^n)
        -
        \nabla_w J_n(w_*^n,\pi^n)
        \right)^{\top}
        (w^{n-1}-w_*^n)\\
        &\qquad
        +
        h^2
        \left|
        \nabla_w J_n(w^{n-1},\pi^n)
        -
        \nabla_w J_n(w_*^n,\pi^n)
        \right|_2^2,
    \end{align}
  where in the first inequality we used the non-expansiveness of the projection operator. By \citet[Theorem~2.1.12]{nesterov2004introductory}, it holds that
    \begin{align}\label{eq:strong_smooth_cocoercive}
        &\left(
        \nabla_w J_n(w,\pi^n)-\nabla_w J_n(\bar w,\pi^n)
        \right)^{\top}
        (w-\bar w)\\
        &\qquad\geq
        \frac{\kappa L}{\kappa+L}|w-\bar w|_2^2
        +
        \frac{1}{\kappa+L}
        \left|
        \nabla_w J_n(w,\pi^n)-\nabla_w J_n(\bar w,\pi^n)
        \right|_2^2.
    \end{align}
    Applying \eqref{eq:strong_smooth_cocoercive} with $w=w^{n-1}$ and $\bar w=w_*^n$ and substituting into the
  right hand side of \eqref{eq:first_non_expansive_calc}, we have
    \begin{align}\label{eq:dropping_negative_terms}
        |w^n-w_*^n|_2^2
        &\leq
        |w^{n-1}-w_*^n|_2^2
        -
        \frac{2h\kappa L}{\kappa+L}
        |w^{n-1}-w_*^n|_2^2\\
        &\quad+\left(
        h^2
        -
        \frac{2h}{\kappa+L}
        \right)
        \left|
        \nabla_w J_n(w^{n-1},\pi^n)
        -
        \nabla_w J_n(w_*^n,\pi^n)
        \right|_2^2\\
        &=
        \left(
        1-\frac{2h\kappa L}{\kappa+L}
        \right)
        |w^{n-1}-w_*^n|_2^2\\
        &\quad+\left(
        h^2
        -
        \frac{2h}{\kappa+L}
        \right)
        \left|
        \nabla_w J_n(w^{n-1},\pi^n)
        -
        \nabla_w J_n(w_*^n,\pi^n)
        \right|_2^2.
    \end{align}
    Since $hL<1$ and $\kappa\leq L$, it holds that $\frac{2L h \kappa}
    {\kappa+L}\geq h \kappa$ and
    \begin{equation}
        h^2
        -
        \frac{2h}{\kappa+L}
        \leq
        \frac{h}{L}
        -
        \frac{2h}{\kappa+L}
        =
        h
        \frac{\kappa-L}{L(\kappa+L)}
        \leq
        0.
    \end{equation}
    Hence, after dropping the nonpositive gradient term in
  \eqref{eq:dropping_negative_terms} and using that $1-\frac{2L h \kappa}
    {\kappa+L} \leq 1 - h \kappa$, we arrive at
    \begin{align}
        |w^n-w_*^n|_2^2
        &\leq
        \left(
        1-\frac{2h\kappa L}{\kappa+L}
        \right)
        |w^{n-1}-w_*^n|_2^2\\
        &\leq
        (1-h\kappa)|w^{n-1}-w_*^n|_2^2.
    \end{align}
  Now adding and subtracting $w_*^{n-1}$ on the right hand side and using Young's inequality for any $\delta > 0$,
  it holds that
    \begin{align}
        |w^n-w_*^n|_2^2
        &\leq
        (1+\delta)(1-h\kappa)
        |w^{n-1}-w_*^{n-1}|_2^2+
        \left( 1+ \frac{1}{\delta}\right)\left(1-h \kappa \right)
        |w_*^{n}-w_*^{n-1}|_2^2.
    \end{align}
    By Lemma \ref{lemma:policy_drift}, for every $n\geq1$ we have
    \begin{align}\label{eq:choose_delta}
        |w^n-w_*^n|_2^2
        &\leq
        (1+\delta)(1-h\kappa)
        |w^{n-1}-w_*^{n-1}|_2^2+
        \lambda C_{\text{drift}}\left( 1+ \frac{1}{\delta}\right)\left(1-h \kappa \right).
    \end{align}
    Now choose $\delta:=\frac{h\kappa}{2(1-h\kappa)}$, which is well-defined because $0<h\kappa\leq hL<1$. Then \eqref{eq:choose_delta} becomes
  \begin{align}\label{eq:actor_tracking_recursion}
      |w^n-w_*^n|_2^2
        &\leq
        \left(
        1+
        \frac{h\kappa}{2(1-h\kappa)}
        \right)(1-h\kappa)|w^{n-1}-w_*^{n-1}|_2^2+
        \lambda C_{\text{drift}}\left(
        1+
        \frac{2(1-h\kappa)}{h\kappa}
        \right) \left(1-h \kappa \right)\\
        &\leq
        \left(
        1-\frac{h\kappa}{2}
        \right)
        |w^{n-1}-w_*^{n-1}|_2^2
        +
        \frac{
        2\lambda C_{\text{drift}}
        }{
        h\kappa
        }.
  \end{align}
  where in the last inequality we used that $\left(1+\frac{2(1-h\kappa)}{h\kappa}\right)(1-h\kappa)
      \leq \frac{2}{h\kappa}$. Iterating \eqref{eq:actor_tracking_recursion}, for any
    $k\in\mathbb N$ gives
    \begin{align}\label{eq:before_average}
        |w^k-w_*^k|_2^2
        &\leq
        \left(
        1-\frac{h\kappa}{2}
        \right)^k
        |w^0-w_*^0|_2^2+
        \frac{
        2\lambda C_{\text{drift}}
        }{
        h\kappa
        }
        \sum_{j=0}^{k-1}
        \left(
        1-\frac{h\kappa}{2}
        \right)^j\\
        &\leq
        \left(
        1-\frac{h\kappa}{2}
        \right)^k
        |w^0-w_*^0|_2^2
        +
        \frac{
        4\lambda C_{\text{drift}}
        }{
        h^2\kappa^2
        }.
    \end{align}
    To ease notation, let $C_{\mathrm{init}}
        :=
        \frac{2|w^0-w_*^0|_2^2}{\kappa}$ and $C_{\mathrm{track}}
        :=
        \frac{4C_{\text{drift}}}{\kappa^2}$.
    Averaging \eqref{eq:before_average} over $k=0,\ldots,n-1$, we obtain
    \begin{align}\label{eq:average_actor_parameter_tracking_simplified}
        \frac{1}{n}
        \sum_{k=0}^{n-1}
        |w^k-w_*^k|_2^2
        &\leq
        \frac{|w^0-w_*^0|_2^2}{n}
        \sum_{k=0}^{n-1}
        \left(
        1-\frac{h\kappa}{2}
        \right)^k
        +
        \frac{\lambda C_{\mathrm{track}}}{h^2}\\
        &\leq
        \frac{C_{\mathrm{init}}}{h n}
        +
        \frac{\lambda C_{\mathrm{track}}}{h^2}.
    \end{align}
    Now fix $k\in\{0,\ldots,n-1\}$. Since $w\mapsto J_k(w,\pi^k)$ is
    $L$-smooth on $\mathfrak{B}(\mathrm R)$ by Theorem \ref{thm:strong_convexity_L_smoothness_objective} for any $k \in \mathbb{N}$ we have    \begin{align}\label{eq:by_L_smoothness}
        J_k(w^k,\pi^k)-J_k(w_*^k,\pi^k)
        &\leq
        \left(\nabla_w J_k(w_*^k,\pi^k)\right)^{\top}
        (w^k-w_*^k)
        +
        \frac{L}{2}|w^k-w_*^k|_2^2.
    \end{align}
    By Lemma \ref{lem:uniform_actor_gradient_bound} for every $k\in\mathbb N$
    and every $w\in\mathfrak{B}(\mathrm R)$ it holds that
    \begin{equation}
        \left|
        \nabla_w J_k(w,\pi^k)
        \right|_2
        \leq 2\sqrt M+
        \sqrt{\frac{2}{\pi}}\frac{\widehat Q_{\max}}{\tau}\sum_{i=1}^M\frac1{\sigma_i} := C_{\text{grad}}
    \end{equation}
    Moreover, since $w^k,w_*^k\in\mathfrak{B}(\mathrm R)$, it holds that $|w^k-w_*^k|_2\leq2\mathrm R$, which in turn implies
  that $|w^k-w_*^k|_2^2\leq2\mathrm R|w^k-w_*^k|_2$. Using this bound in \eqref{eq:by_L_smoothness} we have
    \begin{align}\label{eq:before_square_root}
        J_k(w^k,\pi^k)-J_k(w_*^k,\pi^k)
        &\leq
        C_{\text{grad}}|w^k-w_*^k|_2
        +
        L\mathrm R|w^k-w_*^k|_2\\
        &=
        (C_{\text{grad}}+L\mathrm R)|w^k-w_*^k|_2.
    \end{align}
    Since $w_*^k$ minimises $w\mapsto J_k(w,\pi^k)$ over $\mathfrak{B}(\mathrm R)$ and $w^k\in\mathfrak{B}(\mathrm R)$, the
  left hand side of \eqref{eq:before_square_root} is nonnegative. Define
  $C_3:=(C_{\text{grad}}+L\mathrm R)^{\frac{1}{2}}$. Then taking the square root of \eqref{eq:before_square_root}, averaging and using
  Hölder's inequality, it holds that
    \begin{align}
        \frac{1}{n}
        \sum_{k=0}^{n-1}
        \left(
        J_k(w^k,\pi^k)-J_k(w_*^k,\pi^k)
        \right)^{\frac{1}{2}}
        &\leq
        \frac{C_3}{n}
        \sum_{k=0}^{n-1}
        |w^k-w_*^k|_2^{\frac{1}{2}}\\
        &\leq
        C_3
        \left(
        \frac{1}{n}
        \sum_{k=0}^{n-1}
        |w^k-w_*^k|_2^2
        \right)^{\frac{1}{4}}.
    \end{align}
    Substituting \eqref{eq:average_actor_parameter_tracking_simplified}, we obtain
    \begin{align}
        \frac{1}{n}
        \sum_{k=0}^{n-1}
        \left(
        J_k(w^k,\pi^k)-J_k(w_*^k,\pi^k)
        \right)^{\frac{1}{2}}
        &\leq
        C_3
        \left(
        \frac{C_{\mathrm{init}}}{h n}
        +
        \frac{\lambda C_{\mathrm{track}}}{h^2}
        \right)^{\frac{1}{4}}\\
        &\leq
        C_3\max\left\{C_{\mathrm{init}},C_{\mathrm{track}}\right\}^{\frac14}
        \left(
        \frac{1}{hn}
        +
        \frac{\lambda}{h^2}
        \right)^{\frac14}.
    \end{align}
  Thus the result follows with $C_1
      :=
      C_3\max\left\{C_{\mathrm{init}},C_{\mathrm{track}}\right\}^{\frac14}$.
  \end{proof}

\section{Proof of Theorem \ref{thm:tuned_convergence_rate}}
\begin{proof}
The proof is a direct calculation. Let $\vartheta:=\min\left\{\frac{1}{\tau},\frac{1}{\chi}\right\}$, $\lambda = \frac{\vartheta}{2N^{\frac{4}{5}}}$ and $0 < h < \frac{\tau}{2\tau + R_{\widehat{Q}}}$. Since $N\geq1$ we have $0<\tau\lambda\leq\frac12$ and $0<\chi\lambda\leq\frac12$. 
Therefore by Theorem \ref{thm:main_final_bound} and Lemma \ref{lemma:actor_tracking}, there exists $\mathrm C<\infty$, independent of $N$, such that
\begin{align}
&\min_{0\leq r\leq N-1}
V_\tau^{\pi_{w^r}}(\rho)-V_\tau^{\pi^*}(\rho)\\
&\quad\leq
\mathrm{C}\left(
\frac1{\lambda N}
+\lambda
+\left(\frac1{hN}+\frac{\lambda}{h^2}\right)^{\frac{1}{4}}
+\varepsilon_{\mathrm{approx}}(N)
\right)\\
&\quad=
\mathrm{C}\left(
\frac2\vartheta N^{-\frac{1}{5}}
+\frac\vartheta2N^{-\frac{4}{5}}
+N^{-\frac{1}{5}}\left(\frac1{hN^{\frac{1}{5}}}+\frac\vartheta{2h^2}\right)^{\frac{1}{4}}
+\varepsilon_{\mathrm{approx}}(N)
\right)\\
&\quad\leq
\mathrm{C}\left(
\left(
\frac2\vartheta
+\frac\vartheta2
+\left(\frac1h+\frac\vartheta{2h^2}\right)^{\frac{1}{4}}
\right)N^{-\frac{1}{5}}
+\varepsilon_{\mathrm{approx}}(N)
\right)\\
&\quad\leq
C\left(N^{-\frac{1}{5}}+\varepsilon_{\mathrm{approx}}(N)\right),
\end{align}
where
\begin{equation}
C
:=
\mathrm{C}\max\left\{
\frac2\vartheta
+\frac\vartheta2
+\left(\frac1h+\frac\vartheta{2h^2}\right)^{\frac{1}{4}},
1
\right\}.
\end{equation}
\end{proof}

\section{Proof of Corollary \ref{cor:gibbs_actor_tracking}}
\begin{proof}
Firstly recall that Theorem \ref{thm:strong_convexity_L_smoothness_objective} holds in the case where the target policy is $\pi^{n+1}_{G}$ for all $n \in \mathbb{N}$. Therefore $w \mapsto J_{n}(w,\pi^n_{G})$ is $(1-\gamma)\left(2-\frac{R_{\widehat Q}}{\tau}\right)c_{\Sigma,\mathrm R}\lambda_P$-strongly convex and $2+\frac{R_{\widehat Q}}{\tau}$-smooth for all $w \in \mathfrak{B}(\mathrm{R})$ and $n\geq1$. Moreover, it also directly holds that
\begin{equation}
    \left|\log\frac{d\pi_{\mathrm G}^k}{d\Lambda}\right|_{B_b(S\times A)}
    \leq\frac{2\widehat Q_{\max}}{\tau},
\end{equation}
and consequently by \eqref{eq:KL_difference_bound_by_logdens} we have
\begin{align}
  \operatorname{KL}(\pi_{w^n}(\cdot|s)|\pi_{\mathrm G}^n(\cdot|s))
  &\leq
  K_{\Sigma,\mathrm R} + \left| \log \frac{d\pi^n_G}{d\Lambda}\right|_{B_b(S\times A)}\\
  &\leq K_{\Sigma,\mathrm R} + \frac{2\widehat Q_{\max}}{\tau} \\
  &:=M_{\mathrm G}
\end{align}
Therefore, analogously to the proof of Lemma \ref{lemma:actor_tracking}, let
\begin{equation}
    D_{k}(w)
    :=J_{k+1}(w,\pi_{\mathrm G}^{k+1})-J_k(w,\pi_{\mathrm G}^k),
\end{equation}
then for all $k\geq1$ it holds that
\begin{align}
&\left|
\left(J_{k+1}(w,\pi_{\mathrm G}^{k+1})-J_{k+1}(w,\pi_{\mathrm G}^{k})\right)
-\left(J_{k+1}(w',\pi_{\mathrm G}^{k+1})-J_{k+1}(w',\pi_{\mathrm G}^{k})\right)
\right|\\
&\quad=
\frac1\tau\left|
\int_S\int_A
\left(\widehat Q^k(s,a)-\widehat Q^{k-1}(s,a)\right)
(\pi_w-\pi_{w'})(da|s)\rho_{k+1}(ds)
\right|\\
&\quad\leq\frac{4\widehat Q_{\max}}{\tau}.
\end{align}
Moreover, using $\rho_{k+1}-\rho_k=\chi\lambda(d_\rho^{\pi_{w^k}}-\rho_k)$ exactly as in \eqref{eq:replay_objective_drift_uniform} gives
\begin{align}
&\left|
\left(J_{k+1}(w,\pi_{\mathrm G}^{k})-J_k(w,\pi_{\mathrm G}^{k})\right)
-\left(J_{k+1}(w',\pi_{\mathrm G}^{k})-J_k(w',\pi_{\mathrm G}^{k})\right)
\right| \leq2\chi\lambda M_{\mathrm G}.
\end{align}
Therefore it also holds that
\begin{equation}
    |D_{k}(w)-D_{k}(w')|
    \leq\frac{4\widehat Q_{\max}}{\tau}+2\chi\lambda M_{\mathrm G}.
\end{equation}
Now by strong convexity and optimality, we have
\begin{align}
\frac{\kappa}{2}|w_{\mathrm G,*}^{k+1}-w_{\mathrm G,*}^k|_2^2
&\leq
J_{k+1}(w_{\mathrm G,*}^k,\pi_{\mathrm G}^{k+1})
-J_{k+1}(w_{\mathrm G,*}^{k+1},\pi_{\mathrm G}^{k+1})\\
&\leq
D_{k}(w_{\mathrm G,*}^k)
-D_{k}(w_{\mathrm G,*}^{k+1})\\
&\leq
\frac{4\widehat Q_{\max}}{\tau}+2\chi\lambda M_{\mathrm G}.
\end{align}
Hence after rearranging it holds that
\begin{align}
|w_{\mathrm G,*}^{k+1}-w_{\mathrm G,*}^k|_2^2
&\leq
C_{\mathrm{drift}}
\left(\frac1\tau+\chi\lambda\right),
\end{align}
where $C_{\mathrm{drift}}
    :=\frac1\kappa\max\left\{8\widehat Q_{\max},4M_{\mathrm G}\right\}$.
Now by \eqref{eq:strong_smooth_cocoercive} and \eqref{eq:dropping_negative_terms}, for all $k\geq1$ it holds that
\begin{align}
|w^{k+1}-w_{\mathrm G,*}^{k+1}|_2^2
&\leq
\left(1-\frac{h\kappa}{2}\right)
|w^k-w_{\mathrm G,*}^k|_2^2
+\frac{2C_{\mathrm{drift}}}{h\kappa}
\left(\frac1\tau+\chi\lambda\right).
\end{align}
Since $|w^1-w_{\mathrm G,*}^1|_2\leq2\mathrm R$, iterating and averaging yields
\begin{align}
\frac1n\sum_{k=1}^{n-1}|w^k-w_{\mathrm G,*}^k|_2^2
&\leq
\frac{8\mathrm R^2}{h\kappa n}
+\frac{4C_{\mathrm{drift}}}{h^2\kappa^2}
\left(\frac1\tau+\chi\lambda\right).
\end{align}
Moreover, by Lemma \ref{lem:uniform_actor_gradient_bound}, for every $w\in\mathfrak B(\mathrm R)$ and $k\geq1$ it holds that
\begin{equation}
|\nabla_wJ_k(w,\pi_{\mathrm G}^k)|_2
\leq
C_{\text{grad}}
:=
2\sqrt M+\sqrt{\frac2\pi}\frac{\widehat Q_{\max}}{\tau}\sum_{i=1}^M\frac1{\sigma_i}.
\end{equation}
Thus, by $L$-smoothness and H\"older's inequality,
\begin{align}
&\frac1n\sum_{k=1}^{n-1}
\left(J_k(w^k,\pi_{\mathrm G}^k)-J_k(w_{\mathrm G,*}^k,\pi_{\mathrm G}^k)\right)^{\frac{1}{2}}\\
&\quad\leq
(C_{\text{grad}}+L\mathrm R)^{\frac{1}{2}}
\left(
\frac1n\sum_{k=1}^{n-1}|w^k-w_{\mathrm G,*}^k|_2^2
\right)^{\frac{1}{4}}\\
&\quad\leq
C
\left(
\frac1{hn}
+\frac1{h^2}\left(\frac1\tau+\chi\lambda\right)
\right)^{\frac{1}{4}},
\end{align}
where
\begin{equation}
C
:=
(C_{\text{grad}}+L\mathrm R)^{\frac{1}{2}}
\max\left\{
\frac{8\mathrm R^2}{\kappa},
\frac{4C_{\mathrm{drift}}}{\kappa^2}
\right\}^{\frac{1}{4}}.
\end{equation}

\end{proof}

\section{Auxiliary results}
\subsection{Proof of Lemma \ref{lem:uniform_entropy_bound}}
\setcounter{lemma}{4}
\begin{lemma}[Entropy of a squashed Gaussian policy]
\label{lem:uniform_entropy_bound}
For every $s\in S$ and every $w\in\mathfrak B(\mathrm R)$, it holds that
\begin{align}\label{eq:uniform_entropy_identity}
\mathrm{H}(\pi_w(\cdot|s))
&=
\log\left(\frac{(2\pi)^{\frac{M}{2}}\sqrt{\det\Sigma}}{2^M}\right)+\frac M2
-2\int_{\mathbb R^M}\sum_{i=1}^M\log\cosh(u_i)q_w(du|s) \\
&\geq -2\sqrt M\,\mathrm R-2\sqrt{\frac2\pi}\sum_{i=1}^M\sigma_i
-\left|\log\left(\frac{(2\pi)^{\frac{M}{2}}\sqrt{\det\Sigma}}{2^M}\right)+\frac M2\right| \\
&=:-K_{\Sigma,\mathrm R}.
\end{align}
\end{lemma}
\begin{proof}
Since the Jacobian of $u\mapsto\tanh u$ is $D(\tanh u)$, the change-of-variables formula gives, for every $a\in A$ with $u=\tanh^{-1}a$,
\begin{equation}\label{eq:uniform_actor_density}
\frac{d\pi_w}{d\Lambda}(s,a)
=
\frac{2^M}{(2\pi)^{\frac{M}{2}}\sqrt{\det\Sigma}}
\frac{1}{\det D(a)}
\exp\left(-\frac12(u-m_w(s))^\top\Sigma^{-1}(u-m_w(s))\right).
\end{equation}
Therefore, by Definition \ref{def:pushforward-pullback} it holds that
\begin{align}
&\mathrm{H}(\pi_w(\cdot|s)) \\
&=
-\int_{\mathbb R^M}
\Bigg(
M\log 2-\log\bigl((2\pi)^{\frac{M}{2}}\sqrt{\det\Sigma}\bigr)-\frac12(u-m_w(s))^\top\Sigma^{-1}(u-m_w(s))
\Bigg)q_w(du|s)\\
&\quad+\int_{\mathbb R^M}\sum_{i=1}^M\log(1-\tanh^2u_i)q_w(du|s)\\
&=
\log\left(\frac{(2\pi)^{\frac{M}{2}}\sqrt{\det\Sigma}}{2^M}\right)+\frac M2
-2\int_{\mathbb R^M}\sum_{i=1}^M\log\cosh(u_i)q_w(du|s),
\end{align}
where the second equality uses $1-\tanh^2u_i=\cosh^{-2}u_i$ and $\int_{\mathbb R^M}(u-m_w(s))^\top\Sigma^{-1}(u-m_w(s))q_w(du|s)=M$. Finally, $\log\cosh u_i\leq|u_i|$ and
\begin{align}
\int_{\mathbb R^M}|u|_1q_w(du|s)
&\leq
|m_w(s)|_1+\int_{\mathbb R^M}|u-m_w(s)|_1q_w(du|s)\\
&=
|m_w(s)|_1+\sqrt{\frac2\pi}\sum_{i=1}^M\sigma_i
\leq
\sqrt M\,\mathrm R+\sqrt{\frac2\pi}\sum_{i=1}^M\sigma_i,
\end{align}
where $|m_w(s)|_1\leq\sqrt M\,|x(s)^\top w|_2\leq\sqrt M\,\mathrm R$.
Substitution into \eqref{eq:uniform_entropy_identity} concludes the proof.
\end{proof}

\subsection{Proof of Theorem \ref{thm:main_value_bound}}
\begin{theorem}\label{thm:main_value_bound}
Let Assumption \ref{ass:Q_oracle} hold and suppose that $0<\tau\lambda<1$. Then there exists a non negative constant $C_{\mathrm{KL}} < \infty$ such that for all $n \in \mathbb{N}$ and $s \in S$ it holds that 
\begin{align}\label{eq:KL_continuity}
      \left|
      \mathrm{H}(\pi_{w^n}(\cdot|s))
      -
      \mathrm{H}(\pi^n(\cdot|s))
      \right| \leq
      C_{\mathrm{KL}}
      \left(
      \operatorname{KL}(\pi_{w^n}(\cdot|s)|\pi^n(\cdot|s))
      \right)^{\frac{1}{2}}.
  \end{align}
\end{theorem}
\begin{proof}
 Firstly observe that, by definition of the entropy, for all $s \in S$ and for all $n \geq 0$ we have
  \begin{align}
       &\mathrm{H}(\pi_{w^n}(\cdot|s))
       -
       \mathrm{H}(\pi^n(\cdot|s))\\
       &=
       -
       \int_{A} \log \frac{d\pi_{w^n}}{d\Lambda}(s,a)\pi_{w^n}(da|s)
       +
       \int_{A}\log\frac{d\pi^n}{d\Lambda}(s,a) \pi^n(da|s)\\
       &=
       -
       \int_{A}
       \log \left(
       \frac{d\pi_{w^n}}{d\pi^n}(s,a)
       \frac{d\pi^n}{d\Lambda}(s,a)
       \right)
       \pi_{w^n}(da|s)
       +
       \int_{A}
       \log\frac{d\pi^n}{d\Lambda}(s,a)
       \pi^n(da|s)\\
       &=
       -
       \operatorname{KL}(\pi_{w^n}(\cdot|s)|\pi^n(\cdot|s))
       -
       \int_{A}
       \log \frac{d\pi^n}{d\Lambda}(s,a)
       \left(\pi_{w^n} - \pi^n \right)(da|s).
  \end{align}
  Therefore, by the triangle inequality, it holds that
  \begin{align}\label{eq:triangle_KL}
       &\left|
       \mathrm{H}(\pi_{w^n}(\cdot|s))
       -
       \mathrm{H}(\pi^n(\cdot|s))
       \right|\\
       &\leq
       \operatorname{KL}(\pi_{w^n}(\cdot|s)|\pi^n(\cdot|s))
       +
       \left|
       \int_{A}
       \log \frac{d\pi^n}{d\Lambda}(s,a)
       \left(\pi_{w^n} - \pi^n \right)(da|s)
       \right|.
  \end{align}
Furthermore recall the standard identity
  \begin{align}
      \operatorname{KL}(\pi_{w^n}(\cdot|s)|\pi^n(\cdot|s))
      &=
      \int_{A}\log \frac{d\pi_{w^n}}{d\pi^n}(s,a)\pi_{w^n}(da|s)\\
      &=
      \int_{A}\log \frac{d\pi_{w^n}}{d\Lambda}(s,a)\pi_{w^n}(da|s)
      -
      \int_{A}\log \frac{d\pi^n}{d\Lambda}(s,a)\pi_{w^n}(da|s)\\
      &\leq -\mathrm{H}(\pi_{w^n}(\cdot|s)) + \left| \log \frac{d\pi^n}{d\Lambda}\right|_{B_b(S\times A)}.
  \end{align}
  Since $\pi_{w^n} \in \Pi_\Sigma$ for all $n \geq 0$, Lemma \ref{lem:uniform_entropy_bound} provides the following bound
  \begin{equation}
  -\mathrm{H}(\pi_{w^n}(\cdot|s))
  \leq K_{\Sigma,\mathrm R}.
  \end{equation}
  Therefore, it holds that
  \begin{align}\label{eq:KL_difference_bound_by_logdens}
      \operatorname{KL}(\pi_{w^n}(\cdot|s)|\pi^n(\cdot|s))^{\frac{1}{2}}
      &\leq
      \left(K_{\Sigma,\mathrm R} + \left| \log \frac{d\pi^n}{d\Lambda}\right|_{B_b(S\times A)}\right)^{\frac{1}{2}}.
  \end{align}Multiplying both sides by $\left(
      \operatorname{KL}(\pi_{w^n}(\cdot|s)|\pi^n(\cdot|s))
      \right)^{\frac{1}{2}}
      \geq 0$ we have
\begin{align}\label{eq:square_root_upper_bound}
    \operatorname{KL}(\pi_{w^n}(\cdot|s)|\pi^n(\cdot|s))
      &\leq
      \left(K_{\Sigma,\mathrm R} + \left| \log \frac{d\pi^n}{d\Lambda}\right|_{B_b(S\times A)}\right)^{\frac{1}{2}}\operatorname{KL}(\pi_{w^n}(\cdot|s)|\pi^n(\cdot|s))^{\frac{1}{2}}.
\end{align}
Using Pinsker's inequality we bound the second term of \eqref{eq:triangle_KL} by
  \begin{align}\label{eq:pinsker_used}
       \left|
       \int_{A}
       \log \frac{d\pi^n}{d\Lambda}(s,a)
       \left(\pi_{w^n} - \pi^n \right)(da|s)
       \right|
       &\leq
       \left|\log \frac{d\pi^n}{d\Lambda}\right|_{B_b(S\times A)}
       \left|\pi_{w^n}(\cdot|s)-\pi^n(\cdot|s)\right|_{\mathcal M(A)}\\
       &\leq \sqrt{2}
         \left|\log \frac{d\pi^n}{d\Lambda}\right|_{B_b(S\times A)}\left(
       \operatorname{KL}(\pi_{w^n}(\cdot|s)|\pi^n(\cdot|s))
       \right)^{\frac{1}{2}}.
  \end{align}
  Substituting \eqref{eq:square_root_upper_bound} and \eqref{eq:pinsker_used} into \eqref{eq:triangle_KL}, we arrive at
  \begin{align}\label{eq:KL_continuity_preliminary}
      &\left|
      \mathrm{H}(\pi_{w^n}(\cdot|s))
      -
      \mathrm{H}(\pi^n(\cdot|s))
      \right|\\
      &\qquad\leq
      \left(
      \left(K_{\Sigma,\mathrm R} + \left| \log \frac{d\pi^n}{d\Lambda}\right|_{B_b(S\times A)}\right)^{\frac{1}{2}}
      +
      \sqrt{2}\left| \log \frac{d\pi^n}{d\Lambda}\right|_{B_b(S\times A)}
      \right)
      \left(
      \operatorname{KL}(\pi_{w^n}(\cdot|s)|\pi^n(\cdot|s))
      \right)^{\frac{1}{2}}.
  \end{align}
We now prove that the log densities are uniformly bounded. To that end, by iterating \eqref{eq:mirror_descent} from $\pi^0=\Lambda$ we have
  \begin{align}\label{eq:mirror_descent_closed_form}
      \frac{d\pi^{n+1}}{d\Lambda}(s,a)
      &=
      \frac{1}{ \overline Z_{n+1}(s)}
      \exp\left(
      -\lambda\sum_{k=1}^{n+1}
      (1-\tau\lambda)^{n+1-k}
      \widehat{Q}^{k-1}(s,a)
      \right)
  \end{align} with $ \overline Z_{n+1}(s)=
      \int_A
      \exp\left(
      -\lambda\sum_{k=1}^{n+1}
      (1-\tau\lambda)^{n+1-k}
      \widehat Q^{k-1}(s,a)
      \right)\Lambda(da)$. Therefore, by assumption \ref{ass:Q_oracle} we have $\left| \widehat{Q}^n \right|_{B_b(S\times A)} \leq \widehat Q_{\max}$ thus it also holds that
\begin{align}\label{eq:uniform_target_log_density_bound}
    \left|\log \frac{d\pi^n}{d\Lambda}(s,a)\right|
    &=
    \left|
    -\log \overline Z_n(s)
    -
    \lambda\sum_{k=1}^{n}
    (1-\tau\lambda)^{n-k}
    \widehat Q^{k-1}(s,a)
    \right|\\
    &\leq
    \left|\log \overline Z_n(s)\right|
    +
    \left|
    \lambda\sum_{k=1}^{n}
    (1-\tau\lambda)^{n-k}
    \widehat Q^{k-1}(s,a)
    \right|\\
    &\leq \frac{2\widehat Q_{\max}}{\tau} 
\end{align}
Therefore, substituting this bound into \eqref{eq:KL_continuity_preliminary} gives
  \begin{align}
      &\left|
      \mathrm{H}(\pi_{w^n}(\cdot|s))
      -
      \mathrm{H}(\pi^n(\cdot|s))
      \right|\\
      &\leq
      \left(
      \left(K_{\Sigma,\mathrm R}+\frac{2\widehat Q_{\max}}{\tau}\right)^{\frac{1}{2}}
      +\sqrt2\left(\frac{2\widehat Q_{\max}}{\tau}\right)
      \right)
      \operatorname{KL}(\pi_{w^n}(\cdot|s)|\pi^n(\cdot|s))^{\frac{1}{2}}\\
      &=C_{\mathrm{KL}}
      \operatorname{KL}(\pi_{w^n}(\cdot|s)|\pi^n(\cdot|s))^{\frac{1}{2}},
  \end{align}
with
\begin{equation}
    C_{\mathrm{KL}}:=\left(
      \left(K_{\Sigma,\mathrm R}+\frac{2\widehat Q_{\max}}{\tau}\right)^{\frac{1}{2}}
      +\sqrt2\left(\frac{2\widehat Q_{\max}}{\tau}\right)
      \right).
\end{equation}
\end{proof}

\subsection{Proof of Lemma \ref{lemma:occupancy_L2_bounds}}
\setcounter{lemma}{6}
\begin{lemma}\label{lemma:occupancy_L2_bounds}
      Let $\rho \in \mathcal{P}(S)$ and suppose that there exists a constant $A_2<\infty$ such that, for every $s\in S$ and $a\in A$,
  \begin{equation}
      \left|
      \frac{dP}{d\rho}(s,a,\cdot)
      \right|_{L^2(\rho)}
      \leq
      A_2.
  \end{equation}
  Then for any $\pi \in \mathcal{P}(A|S)$ it holds that
      \begin{equation}
          \left|
         \frac{d d_{\rho}^{\pi}}{d\rho}
         \right|_{L^2(\rho)}
         \leq (1-\gamma)+\gamma A_2
      \end{equation}
  \end{lemma}
  \begin{proof}
  Firstly we seek to identify the density of $P_{\pi}$ with respect to $\rho$ for any $\pi \in \mathcal{P}(A|S)$.
    To that end, for any $B \in \mathcal{B}(S)$ and any $s\in S$ we have
    \begin{align}
        P_{\pi}(B|s)
        &=
        \int_{A}P(B|s,a)\pi(da|s) \\
        &=
        \int_{A}
        \int_{B}
        \frac{d P}{d\rho}(s,a,s')
        \rho(ds')
        \pi(da|s) \\
        &=
        \int_B
        \left(
        \int_A
        \frac{dP}{d\rho}(s,a,s')
        \pi(da|s)
        \right)
        \rho(ds').
    \end{align}
    Since this holds for all $B \in \mathcal{B}(S)$, we have
    \begin{equation}
        \frac{dP_{\pi}}{d\rho}(s,s')
        =
        \int_A
        \frac{dP}{d\rho}(s,a,s')
        \pi(da|s).
    \end{equation}
    Therefore, by Minkowski's integral inequality it holds that for all $s\in S$,
    \begin{align}\label{eq:bound_on_L2_P_pi}
        \left|
        \frac{dP_{\pi}}{d\rho}(s,\cdot)
        \right|_{L^2(\rho)}
        &=
        \left(
        \int_S
        \left|
        \int_A
        \frac{dP}{d\rho}(s,a,s')
        \pi(da|s)
        \right|^2
        \rho(ds')
        \right)^{\frac{1}{2}}\\
        &\leq
        \int_A
        \left(
        \int_S
        \left|
        \frac{dP}{d\rho}(s,a,s')
        \right|^2
        \rho(ds')
        \right)^{\frac{1}{2}}
        \pi(da|s)\\
        &\leq
        A_2.
    \end{align}

    We now prove a bound that holds for any initial state distribution $\rho\in\mathcal P(S)$. Fix $n\geq1$. We wish to
  identify the density of $\int_S P_\pi^n(\cdot|s)\rho(ds)$ with respect to $\rho$. To that end, observe that for
  any $B\in\mathcal B(S)$,
    \begin{align}
        \int_S P_\pi^n(B|s)\rho(ds)
        &=
        \int_S
        \int_S
        P_\pi(B|s')
        P_\pi^{n-1}(ds'|s)
        \rho(ds)\\
        &=
        \int_S
        P_\pi(B|s')
        \left(
        \int_S
        P_\pi^{n-1}(ds'|s)
        \rho(ds)
        \right) \\
        &=
        \int_S
        \int_B
        \frac{dP_\pi}{d\rho}(s',s'')
        \rho(ds'')
        \left(
        \int_S
        P_\pi^{n-1}(ds'|s)
        \rho(ds)
        \right)\\
        &=
        \int_B
        \left(
        \int_S
        \frac{dP_\pi}{d\rho}(s',s'')
        \left(
        \int_S
        P_\pi^{n-1}(ds'|s)
        \rho(ds)
        \right)
        \right)
        \rho(ds'').
    \end{align}
    Since this holds for all $B\in\mathcal B(S)$, it follows that
    \begin{equation}\label{eq:any_eta_density}
        \frac{d\left(\int_S P_\pi^n(\cdot|s)\rho(ds)\right)}{d\rho}(s'')
        =
        \int_S
        \frac{dP_\pi}{d\rho}(s',s'')
        \left(
        \int_S
        P_\pi^{n-1}(ds'|s)
        \rho(ds)
        \right).
    \end{equation}
    Consequently, by Minkowski's integral inequality and \eqref{eq:bound_on_L2_P_pi}, it holds that
    \begin{align}\label{eq:any_eta_L2_bound}
        \left|
        \frac{d\left(\int_S P_\pi^n(\cdot|s)\rho(ds)\right)}{d\rho}
        \right|_{L^2(\rho)}
        &=
        \left(
        \int_S
        \left|
        \int_S
        \frac{dP_\pi}{d\rho}(s',s'')
        \left(
        \int_S
        P_\pi^{n-1}(ds'|s)
        \rho(ds)
        \right)
        \right|^2
        \rho(ds'')
        \right)^{\frac{1}{2}}\\
        &\leq
        \int_S
        \left|
        \frac{dP_\pi}{d\rho}(s',\cdot)
        \right|_{L^2(\rho)}
        \left(
        \int_S
        P_\pi^{n-1}(ds'|s)
        \rho(ds)
        \right)\\
        &\leq
        A_2
        \int_S
        \left(
        \int_S
        P_\pi^{n-1}(ds'|s)
        \rho(ds)
        \right)\\
        &=
        A_2,
    \end{align}
    where in the final equality we used that $\int_S P_\pi^{n-1}(ds'|s)\rho(ds)$ is a probability measure on $S$. Now recall that by definition, we have
    \begin{align}
        d_{\rho}^{\pi}(ds')
        &=
        (1-\gamma)
        \sum_{n=0}^{\infty}
        \gamma^n
        \int_S P_\pi^n(ds'|s)\rho(ds),
    \end{align}
    with $P_\pi^0(ds'|s)=\delta_s(ds')$. Hence applying the triangle inequality in $L^2(\rho)$ to the partial sums and then taking the limit, we obtain
    \begin{align}
        \left|
        \frac{d d_{\rho}^{\pi}}{d\rho}
        \right|_{L^2(\rho)}
        &\leq
        (1-\gamma)
        \left|
        \frac{d\rho}{d\rho}
        \right|_{L^2(\rho)}
        +
        (1-\gamma)
        \sum_{n=1}^{\infty}
        \gamma^n
        \left|
        \frac{d\left(\int_S P_\pi^n(\cdot|s)\rho(ds)\right)}{d\rho}
        \right|_{L^2(\rho)}\\
        &\leq
        (1-\gamma)
        +
        (1-\gamma)
        \sum_{n=1}^{\infty}
        \gamma^n A_2\\
        &=
        (1-\gamma)+\gamma A_2.
    \end{align}
  \end{proof}

\subsection{Proof of Lemma \ref{lem:uniform_actor_gradient_bound}}
\begin{lemma}[Uniform boundedness of the actor gradient]
      \label{lem:uniform_actor_gradient_bound}
      Let Assumption \ref{ass:Q_oracle} hold. Then for all $n \in \mathbb{N}$ and all $w \in \mathfrak{B}(\mathrm R)$, it holds that 
      \begin{equation}
          \left|\nabla_wJ_{n+1}(w,\pi^{n+1})\right|_2
          \leq
          2\sqrt M+
          \sqrt{\frac{2}{\pi}}
          \frac{\widehat Q_{\max}}{\tau}\sum_{i=1}^M\frac{1}{\sigma_i}.
      \end{equation}
      Moreover, for any $n \in \mathbb{N}$, the same bounds holds for $ \left|\nabla_wJ_{n+1}(w,\pi^{n+1}_{G})\right|_2$.
    \end{lemma}

  \begin{proof}
  Firstly observe that for every $n\in\mathbb N$ and every $w\in\mathfrak{B}(\mathrm R)$, we can write
    \begin{align}\label{eq:two_terms}
        J_{n+1}(w,\pi^{n+1})
        &= \int_{S} \KL(\pi_w(\cdot|s)|\pi^{n+1}(\cdot|s))\rho_{n+1}(ds)\\
        &=\int_{S} \int_{A} \log \frac{d\pi_{w}}{d\pi^{n+1}}(s,a)\pi_w(da|s)\rho_{n+1}(ds) \\
        &= \int_{S} \int_{A} \log\left(\frac{d\pi_{w}}{d\Lambda}(s,a)\frac{d\Lambda}{d\pi^{n+1}}(s,a)\right)\pi_w(da|s)\rho_{n+1}(ds)\\
        &=
        -\int_S
        \mathrm{H}(\pi_w(\cdot|s))
        \rho_{n+1}(ds)-
        \int_S
        \int_A
        \log\frac{d\pi^{n+1}}{d\Lambda}(s,a)
        \pi_w(da|s)
        \rho_{n+1}(ds).
    \end{align}
    We now bound the gradients of these two terms separately. For the first term, using Lemma \ref{lem:uniform_entropy_bound}, for every $s\in S$ we have
    \begin{align}
        -\mathrm{H}(\pi_w(\cdot|s))
        &=
        2\int_{\mathbb R^M}\sum_{i=1}^M\log\cosh(u_i)q_w(du|s)
        -\log\left(\frac{(2\pi)^{\frac{M}{2}}\sqrt{\det\Sigma}}{2^M}\right)-\frac M2\\
        &=
        2\int_{\mathbb R^M}\sum_{i=1}^M\log\cosh(m_w(s)_i+\sigma_i z_i)q^1(dz)
        -\log\left(\frac{(2\pi)^{\frac{M}{2}}\sqrt{\det\Sigma}}{2^M}\right)-\frac M2.
    \end{align}
    Therefore, differentiating under the integral, it follows that
    \begin{equation}
        \nabla_w
        \mathrm{H}(\pi_w(\cdot|s))
        =
        -2\int_{\mathbb R^M}
        x(s)\tanh(m_w(s)+\Sigma^{\frac{1}{2}}z)q^1(dz).
    \end{equation}
    Hence, using $|\tanh u|_2\leq\sqrt M$ and $|x(s)|_{\op}\leq1$ for all $s\in S$, it holds that
    \begin{equation}\label{eq:first_actor_gradient_bound}
        \left|\nabla_w \mathrm{H}(\pi_w(\cdot|s))\right|_2
        \leq
        2\sqrt M.
    \end{equation}

    Before proceeding with the second term, let us simplify the notation with
    \begin{equation}
        F_n(s,u)
        :=
        \lambda\sum_{k=0}^n(1-\tau\lambda)^{n-k}\widehat Q^k(s,\tanh u),
    \end{equation}
    and thus recall that
    \begin{equation}
    \frac{d\pi^{n+1}}{d\Lambda}(s,\tanh u)
    =
    \frac{1}{\overline Z_{n+1}(s)}\exp\left(-F_n(s,u)\right).
    \end{equation}
    Let $\Lambda_{\mathbb R^M}$ denote Lebesgue measure on $\mathbb R^M$. By Definition \ref{def:pushforward-pullback} it holds that
  \begin{align}
      &\int_A
      \log\frac{d\pi^{n+1}}{d\Lambda}(s,a)\pi_w(da|s)
      \\
      &=  \int_A
      \log\frac{d\pi^{n+1}}{d\Lambda}(s,a)(\tanh)_{\#}q_w(da|s)\\
      &=\int_{\mathbb R^M}
      \log\frac{d\pi^{n+1}}{d\Lambda}(s,\tanh u)q_w(du|s) \\
      &=\int_{\mathbb R^M}
      \left(-F_n(s,u)-\log\overline Z_{n+1}(s)\right) q_w(du|s) \\
      &\begin{aligned}
      &=\frac{1}{(2\pi)^{\frac M2}\sqrt{\det\Sigma}}
      \int_{\mathbb R^M}\left(-F_n(s,u)-\log\overline Z_{n+1}(s)\right)\\
      &\qquad\qquad
      \exp\left(-\frac12(u-m_w(s))^\top\Sigma^{-1}(u-m_w(s))\right)
      \Lambda_{\mathbb R^M}(du).
      \end{aligned}
  \end{align}
Therefore differentiation under the integral also gives 
  \begin{align}
      &\nabla_w
      \int_A
      \log\frac{d\pi^{n+1}}{d\Lambda}(s,a)\pi_w(da|s)
      =
      -\int_{\mathbb R^M}
      F_n(s,u)
      x(s)\Sigma^{-1}(u-m_w(s))
      q_w(du|s).
  \end{align}
  Here the term containing $\log\overline Z_{n+1}(s)$ vanishes because $\int_{\mathbb R^M}q_w(du|s)=1$ for every $w\in\mathfrak B(\mathrm R)$.
  Hence for all $s \in S$ it holds that
    \begin{align}
        \left|
        \nabla_w
        \int_A
        \log\frac{d\pi^{n+1}}{d\Lambda}(s,a)\pi_w(da|s)
        \right|_2
        &\leq
        |F_n|_{B_b(S\times\mathbb R^M)}
        \int_{\mathbb R^M}
        |\Sigma^{-1}(u-m_w(s))|_2
        q_w(du|s) \\
        &\leq
        \frac{\widehat Q_{\max}}{\tau}
        \int_{\mathbb R^M}
        |\Sigma^{-1}(u-m_w(s))|_2
        q_w(du|s),
    \end{align}
   Since $q_w(\cdot|s)$ is Gaussian with mean $m_w(s)$ and covariance $\Sigma$, the Gaussian first moment \citep[Ch.~13]{johnson1994continuous} yields
  \begin{align}
      \int_{\mathbb R^M}
      |\Sigma^{-1}(u-m_w(s))|_2q_w(du|s)
      &\leq
      \sum_{i=1}^M\frac{1}{\sigma_i^2}
      \int_{\mathbb R^M}|u_i-m_w(s)_i|q_w(du|s)\\
      &=
      \sqrt{\frac{2}{\pi}}\sum_{i=1}^M\frac{1}{\sigma_i}.
  \end{align}
  Hence it holds that
  \begin{equation}\label{eq:log_density_gradient_bound}
      \left|
      \nabla_w
      \int_A
      \log\frac{d\pi^{n+1}}{d\Lambda}(s,a)\pi_w(da|s)
      \right|_2
      \leq
      \sqrt{\frac{2}{\pi}}\frac{\widehat Q_{\max}}{\tau}\sum_{i=1}^M\frac{1}{\sigma_i}.
  \end{equation}

Combining \eqref{eq:first_actor_gradient_bound} and \eqref{eq:log_density_gradient_bound}, and using that $\rho_{n+1}$ is a probability measure, gives
      \begin{align}
          \left|\nabla_wJ_{n+1}(w,\pi^{n+1})\right|_2
          &\leq
          2\sqrt M+
          \sqrt{\frac{2}{\pi}}\frac{\widehat Q_{\max}}{\tau}\sum_{i=1}^M\frac{1}{\sigma_i}.
      \end{align}
      To bound $ \left|\nabla_wJ_{n+1}(w,\pi^{n+1}_{G})\right|_2$, the same arguments hold with $F_n(s,u) = \frac{\widehat{Q}^n(s, \tanh u)}{\tau} \leq \frac{\widehat{Q}_{\max}}{\tau}$.
  \end{proof}
\setcounter{lemma}{9}

\section{Implementation details}\label{sec:implementation_details}
In this section we give an overview on the environment used in Figures \ref{fig:critic_curvature_performance} and \ref{fig:actor_tracking_lambda}. In order to validate the theory, we restrict to an MDP in which the critics can be solved exactly. To that end, we require that for all $s \in S$, $c(s,\cdot) : A \mapsto \mathbb{R}$ is quadratic and the transition kernel is affine in the action variable. To ease notation, let $g :S \mapsto \mathbb{R}^M$ be $g(s) = f(s)v$ with $f(s) = -\frac25+\frac45\frac{s}{|S|-1}$ and $v = (1,-\frac12,1,-\frac12,\ldots)\in\mathbb R^M$. Denoting $\bar{a} = \frac{1}{M}\sum_{i=1}^{M}a_i$ with $a = (a_1,a_2,\dots,a_M) \in (-1,1)^M$ and $s^+:=\min(s+1,|S|-1)$, $s^-:=\max(s-1,0)$, the cost function and transition dynamics are defined as
\begin{equation}
c(s,a)=\frac{1}{10}\left|a-g(s)\right|_2^2+\frac{3}{100(M-1)}\sum_{i<j}(a_i-a_j)^2+\frac{3}{10}\frac{s}{|S|-1},
\end{equation}
\begin{equation}
    P(s'|s,a) =
\frac{1}{10|S|}
+\frac{9}{10}
\left(
\frac{1+\bar a}{2}\delta_{s^+}(s')
+\frac{1-\bar a}{2}\delta_{s^-}(s')
\right)
\end{equation}
Given that the state space is tabular, we parametrise the actor using the canonical basis function on $\mathbb{R}^{|S|}$, that is $x(s) = e_s \otimes I_{M}$ so that $m_w(s)=x(s)^\top w$, $|x(s)|_{\op}=1$ which in turn implies that $\lambda_{P} = \frac{1}{|S|}$. We take $\Sigma = 0.15^2 I_{M}$ and define the projection radius in Algorithm \ref{algo:SAC_fixed_replay} as $\mathrm{R} = \sqrt{|S|M}$.

In Figure \ref{fig:critic_curvature_performance}, to compare the effect on Algorithm \ref{algo:SAC_fixed_replay} in the regimes where the condition of Theorem \ref{thm:strong_convexity_L_smoothness_objective} is satisfied and not, we firstly learn the exact critic through a linear solve (Theorem \ref{thm:dynamics_programming}), then add perturbations such that both resulting perturbed critics have the same approximation error $\varepsilon_{\mathrm{crit}}$ but different curvatures measures through $R_{\widehat{Q}}$. Concretely, we construct two sequences of critics such that
\begin{equation}
\widehat Q^n(s,a)
=
Q^{\pi_{w^n}}_\tau(s,a)
+\varepsilon_{\mathrm{crit}}\sum_{j=1}^{4}\alpha_j\cos\left(\omega_j^\top a\right)
\end{equation}
with amplitudes $\alpha=(0.55,0.20,0.15,0.10)$, $\sum_{j=1}^4\alpha_j=1$, and such that $\omega_j=\frac{r_j}{\sqrt M}\mathbf 1_M$. Then we construct two instances of $r \in \mathbb{R}^{4}$, with $r=(0.6,0.8,1,1.2)$ and $r=(6,7,8,9)$
giving $\frac{\mathrm{R}_{\widehat{Q}}}{2}<\tau$ and $\frac{\mathrm{R}_{\widehat{Q}}}{2}>\tau$ respectively, while $\left|\widehat Q^n-Q^{\pi_{w^n}}_\tau\right|_{B_b(S\times A)}=\varepsilon_{\mathrm{crit}}$ in both.

In Figure \ref{fig:actor_tracking_lambda}, we use the same MDP with varying dimensions but with exact critics ($\varepsilon_{\mathrm{crit}} = 0$) in order to isolate the effect of the proximal step-size/ penalty $\lambda$ on the actor tracking error. Finally, in all instances, we set the initial state distribution to $\rho(s) = \frac{1}{|S|}$, the entropy temperature to $\tau = 1$, discount factor $\gamma = 0.9$, the actor step size to $h= 0.01$ and the exponential moving average weight in the replay buffer dynamics to $\chi = 1$.
\end{document}